%% file: main.tex
\documentclass{article}

\PassOptionsToPackage{numbers, compress}{natbib}
\usepackage[preprint]{neurips_2026}

\usepackage[utf8]{inputenc}
\usepackage[T1]{fontenc}
\usepackage{hyperref}
\usepackage{url}
\usepackage{booktabs}
\usepackage{graphicx}
\graphicspath{{figures/}}
\usepackage{amsfonts}
\usepackage{amsmath}
\usepackage{amssymb}
\usepackage{amsthm}
\usepackage{nicefrac}
\usepackage{microtype}
\usepackage{xcolor}
\usepackage{enumitem}
\usepackage{subcaption}
\usepackage{wrapfig}
\usepackage{longtable}
\usepackage{multirow}

\newtheorem{lemma}{Lemma}
\newtheorem{proposition}{Proposition}

\theoremstyle{definition}

\newtheorem{remark}{Remark}

\title{LoMETab: Beyond Rank-1 Ensembles\newline for Tabular Deep Learning}

\author{%
  Changryeol Choi\thanks{Equal contribution.} \\
  CJ Logistics \\
  \texttt{changryeol.choi@cj.net} \\
  \And
  Hyewon Park\footnotemark[1] \\
  CJ Logistics \\
  \texttt{hyewon.park16@cj.net} \\
  \And
  Yujin Kwon \\
  CJ Logistics \\
  \texttt{yujin.kwon5@cj.net} \\
  \And
  Gowun Jeong\thanks{Corresponding author.} \\
  CJ Logistics \\
  \texttt{gowun.jeong@cj.net} \\
}

\begin{document}
\maketitle

\begin{abstract}
Recent tabular learning benchmarks increasingly show a tight performance cluster
rather than a clear hierarchy among leading methods, spanning gradient boosted decision trees, attention-based architectures, and implicit ensembles such as TabM.
As benchmark gains plateau, a complementary goal is to understand and control the mechanisms that make simple neural tabular models competitive.
We propose \textbf{LoMETab}, a rank-$r$ generalization of multiplicative implicit ensembles. LoMETab lifts the rank-1 BatchEnsemble/TabM modulation to a rank-$r$ identity-residual Hadamard family
by parameterizing each member weight as $W_k=W\odot(1+A_kB_k^\top)$, where $W$ is shared and $(A_k,B_k)$ are member-specific low-rank factors.
This exposes two practical diversity-control axes: the adapter rank $r$ and the initialization scale $\sigma_{\mathrm{init}}$,
and we prove that for $r \ge 2$ this generalization strictly enlarges BatchEnsemble's hypothesis class.
Empirically, we  show that this added capacity manifests as measurable predictive diversity after training: on representative classification datasets, 
LoMETab sustains higher pairwise KL than an additive low-rank ablation, and $(r,\sigma_{\mathrm{init}})$ provides broad control over pairwise KL, varying by up to several orders of magnitude across configurations. 
The induced diversity is reflected in task-appropriate output-level measures: argmax disagreement for classification and ambiguity for regression, indicating that the control extends beyond pairwise KL to decision- and output-level member variation.
Finally, experiments sweeping over adapter rank $r$ and initialization scale $\sigma_{\mathrm{init}}$ reveal that predictive performance is dataset-dependent over the $(r,\sigma_{\mathrm{init}})$ grid, supporting LoMETab as a controllable family of implicit ensembles rather than a fixed rank-1 construction.

\end{abstract}

\section{Introduction}
\label{sec:introduction}

Tabular data remains one of the most prevalent data modalities in real-world applications, including finance, healthcare, and logistics.
Although gradient boosted decision trees (GBDTs) remain exceptionally strong on such tasks~\citep{chen2016xgboost, prokhorenkova2018catboost}, deep tabular models offer a complementary advantage: their differentiable, modular backbones make it natural to incorporate learned feature embeddings, task-specific heads, auxiliary losses, and parameter-sharing mechanisms within a single end-to-end training objective.
Recent benchmarks show that simple MLP-based tabular models can be surprisingly competitive, often matching or outperforming more complex attention- and retrieval-based architectures~\citep{gorishniy2024modernnca, grinsztajn2022tree, mcelfresh2024neural}.
This makes strengthening simple neural tabular backbones a practical alternative to increasing architectural complexity.

Deep ensembles~\citep{lakshminarayanan2017simple} are a strong way to improve such models, but training and maintaining $K$ independent networks scales linearly with ensemble size.
Implicit ensembles reduce this cost by sharing most parameters among members while retaining member-specific perturbations.
BatchEnsemble~\citep{wen2020batchensemble} assigns two learnable vectors $r_k$ and $s_k$ to each member, which multiplicatively modulate a shared weight $W$ to form an effective member weight 
$W_k = W \odot (s_k r_k^\top)$, yielding ensemble-like behavior at near single-model parameter cost. TabM~\citep{gorishniy2025tabm} applies this idea to tabular MLPs, achieving performance competitive with strong tabular baselines.

However, TabM inherits the fixed rank-1 structure of BatchEnsemble:
each member modulates the shared weight only through a rank-1 mask, leaving no explicit axis for increasing the capacity of member-specific modulation.
Consequently, ensemble diversity---a key driver of ensemble gains---is not directly controllable; it must emerge implicitly from fixed rank-1 masks and training dynamics.
Prior work also reports that BatchEnsemble members can converge to near-identical functions in some settings~\citep{zamyatin2026batchensemble}, suggesting that the rank-1 modulation may not always realize meaningful diversity.

We propose \textbf{LoMETab}, a rank-$r$ generalization of multiplicative implicit ensembles.
Instead of the rank-1 BatchEnsemble mask, LoMETab assigns each member $k$ a low-rank adapter pair $(A_k,B_k)$ and defines its effective weight as $W_k=W\odot(1+A_kB_k^\top)$.
Here, $A_kB_k^\top$ defines a rank-at-most-$r$ residual, and $1 + A_kB_k^\top$ forms an identity-residual multiplicative perturbation of the shared weight.
This Hadamard identity-residual form follows the multiplicative reparameterization used in HiRA \cite{huang2025hira}, but we use it for a different purpose: member-specific diversity generation in scratch-trained tabular implicit ensembles rather than single-model fine-tuning.
The resulting framework exposes the adapter rank $r$ and initialization scale $\sigma_{\mathrm{init}}$ as practical axes for controlling ensemble diversity.


We first establish that the proposed rank axis genuinely expands the rank-1 design space:
for $r \ge 2$, every BatchEnsemble mask can be represented by LoMETab structure because $s_kr_k^\top-\mathbf{1}$ has rank at most two. 
Thus, LoMETab strictly contains BatchEnsemble at the layer-wise effective-weight level, while the converse does not hold.
Expressivity alone, however, does not guarantee meaningful ensemble diversity after training.
We therefore test whether $(r,\sigma_{\mathrm{init}})$ yields controllable differences in member predictions, and how these axes affect performance across datasets.



Our contributions are summarized as follows:
(1) We introduce LoMETab, a rank-$r$ multiplicative implicit-ensemble framework for tabular MLPs, where each member-specific deviation from a shared weight is parameterized by an identity-residual low-rank Hadamard perturbation.
(2) We show that the added rank capacity manifests as measurable predictive diversity after training: LoMETab sustains higher pairwise KL than an additive low-rank ablation,
and pairwise KL spans a wide range across the $(r,\sigma_{\mathrm{init}})$ grid.
The same variation is reflected in argmax disagreement for classification and ambiguity for regression.
(3) Across tabular benchmarks, LoMETab remains competitive with strong tabular baselines while exposing practical trade-offs among rank, ensemble size, and initialization. Tuned configurations frequently select $r>1$ (See App.~\ref{app:n_parameters}),
controlled $r$- and $\sigma_{\mathrm{init}}$-sweeps against TabM $K$-scaling show dataset-dependent effects on performance.

\vspace{-1em}

\section{Related Work}
\label{sec:related}


\textbf{Implicit ensembles and their limitations.}
Implicit ensembles train a single network that internally behaves like a collection of submodels, by sharing most parameters while introducing member-specific perturbations.
BatchEnsemble~\citep{wen2020batchensemble} uses rank-1 multiplicative masks, and TabM~\citep{gorishniy2025tabm} adapts this idea to tabular MLPs.
Other implicit-ensemble formulations diversify members through structural or activation-level variations: Masksembles~\citep{durasov2021masksembles} uses member-specific binary subnetworks, while
FiLMEnsemble~\citep{turkoglu2022filmensemble} applies member-specific feature-wise affine modulation to intermediate activations.
LoMETab differs by replacing the fixed rank-1 multiplicative mask with an identity-residual low-rank multiplicative mask, making the adapter rank a tunable design axis. 


\textbf{Low-rank adaptation and multiplicative reparameterization.}
LoRA~\citep{hu2022lora} introduced additive low-rank updates for parameter-efficient fine-tuning as follows: $W' = W_0 + AB^\top$.
LoRA-Ensemble~\citep{halbheer2024lora} uses low-rank adapters for efficient ensembles on pretrained Vision Transformers.
Multiplicative low-rank reparameterizations have also been explored for single-model fine-tuning:
HiRA~\citep{huang2025hira} uses a Hadamard-product form $W_0 \odot (\mathbf{1} + L_1 L_2)$ to enable higher-rank updates.
Our work adapts this multiplicative identity-residual structure from single-model fine-tuning to scratch-trained tabular implicit ensembles, where member-specific low-rank residuals are used to generate ensemble diversity.

\section{Proposed Method}
\label{sec:method}

In this section, we present LoMETab. Sec.~\ref{sec:method:prelim} fixes notation and reviews the preliminaries of our work. Sec.~\ref{sec:method:lometab} defines the LoMETab layer and establishes its expressivity guarantee.

Sec.~\ref{sec:method:objective} introduces the training objective of LoMETab.

\subsection{Preliminaries}
\label{sec:method:prelim}

\paragraph{Notation.}
We consider an implicit ensemble with $K$ members indexed by $k \in \{1,\ldots,K\}$, built on an $L$-block MLP backbone indexed by $l \in \{1,\ldots,L\}$.
Let $h_k^{(l-1)}$ denote the input representation to block $l$ for member $k$, with $h_k^{(0)}$ being the input $x$ after a piecewise linear embedding layer~\citep{gorishniy2022embeddings}.
At block $l$, all members share a base weight $W^{(l)} \in \mathbb{R}^{d_{\mathrm{out}}^{(l)} \times d_{\mathrm{in}}^{(l)}}$, and member $k$ uses its own adapter parameters $A_k^{(l)}, B_k^{(l)}$ to form an effective weight
$W_k^{(l)}$.
The linear output of member $k$ at block $l$ is $z_k^{(l)}$. 
The block output $h_k^{(l)}$ is obtained after the nonlinearity and dropout.
After the final block, a member-specific prediction head is applied to produce $o_1,\ldots,o_K$
and the final ensemble prediction $\hat{y}$ is obtained by averaging the $K$ member predictions.
When discussing a single generic layer, we omit the block index $l$.
We also omit bias terms for brevity; the same construction applies with or without them.


\textbf{Multiplicative rank-1 implicit ensembles.} BatchEnsemble~\citep{wen2020batchensemble} and TabM~\citep{gorishniy2025tabm} assign each submodel two learnable adapter vectors $r_k \in \mathbb{R}^{d_{\text{in}}}$ and $s_k \in \mathbb{R}^{d_{\text{out}}}$, defining the submodel's weight as $W_k = W \odot (s_k r_k^\top)$,
where $\odot$ denotes the Hadamard (element-wise) product and $s_k r_k^\top \in \mathbb{R}^{d_{\text{out}} \times d_{\text{in}}}$ is a rank-1 multiplicative mask.

\subsection{LoMETab}
\label{sec:method:lometab}

\begin{wrapfigure}{r}{0.55\linewidth}
\vspace{-2em}
\centering
\includegraphics[width=\linewidth]{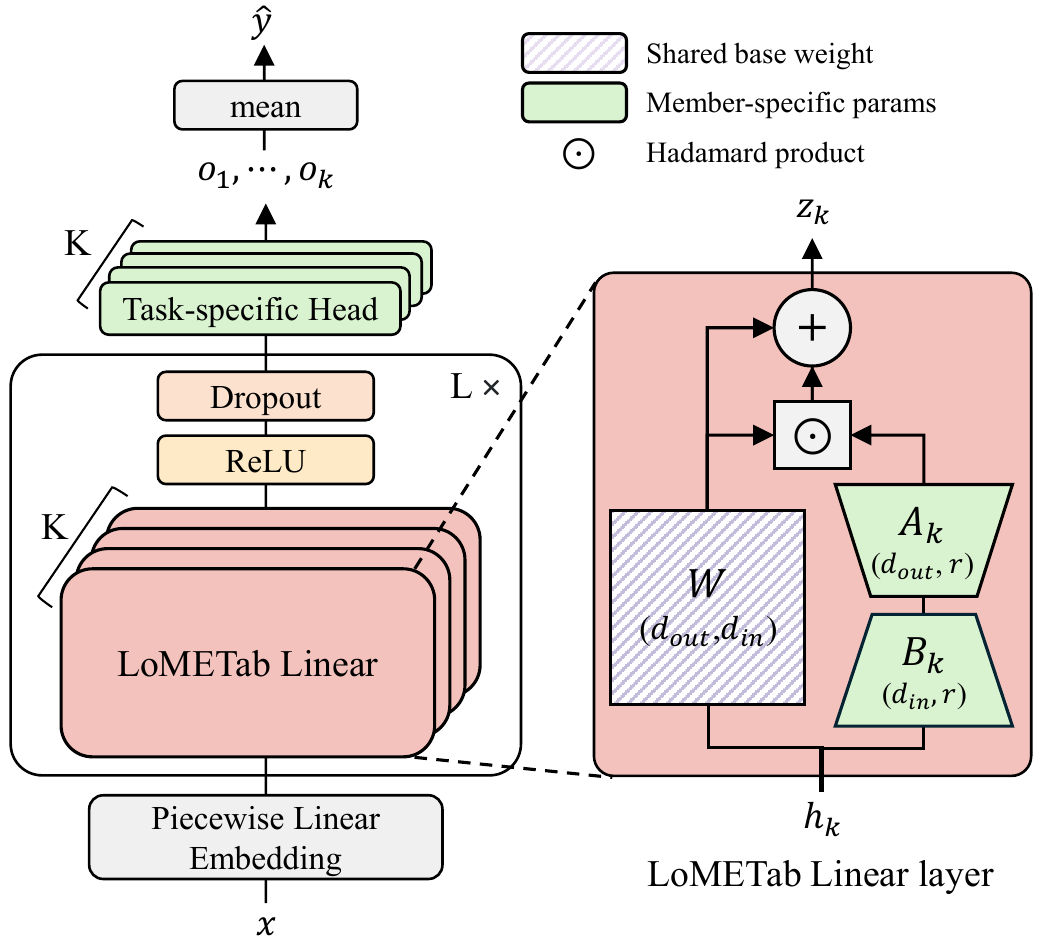}
\vspace{-1em}
\caption{\textbf{LoMETab Architecture.} The block index $l$ is omitted for readability.}
\label{fig:architecture}
\vspace{-1em}
\end{wrapfigure}

\paragraph{Formulation.}
For each member $k$, we assign member-specific low-rank adapter matrices
$A_k \in \mathbb{R}^{d_{\mathrm{out}} \times r}$ and
$B_k \in \mathbb{R}^{d_{\mathrm{in}} \times r}$, so that
$A_kB_k^\top$ forms a low-rank perturbation matrix. The effective weight of
member $k$ is defined as
\begin{equation}
    \label{eq:multiplicative_lora_method}
    W_k = W \odot (\mathbf{1} + A_kB_k^\top).
\end{equation}
Here, $A_kB_k^\top$ is a residual with rank at most $r$, and $1+A_kB_k^\top$ forms an identity-residual multiplicative perturbation of the shared weight. When $A_kB_k^\top\approx 0$, member $k$ remains close to the shared weight $W$; larger residuals induce member-specific deviations.

\paragraph{Architecture.}
Fig.~\ref{fig:architecture} illustrates how LoMETab is inserted into a tabular MLP.
A shared piecewise linear embedding~\citep{gorishniy2022embeddings} is applied once to the input $x$ and the resulting representation is copied across $K$ members.
Each hidden linear layer is then replaced by a \emph{LoMETab Linear layer}: a shared base weight $W$ is combined with member-specific adapters $(A_k,B_k)$ to produce $K$ member-specific effective weights (Eq.~\eqref{eq:multiplicative_lora_method}). 
The output of each member is passed through ReLU and dropout before the next block. After the final block, member-specific prediction heads produce $o_1,\ldots,o_K$, and the ensemble output is obtained by averaging member predictions.


\paragraph{Initialization scale.}

We initialize both $A_k$ and $B_k$ from $\mathcal{N}(0,\sigma_{\mathrm{init}}^2)$.
Small $\sigma_{\mathrm{init}}$ keeps members close to the shared weight at initialization,
whereas larger values induce stronger member-specific perturbations.
Since the effective magnitude of $A_kB_k^\top$ depends on both $r$ and $\sigma_{\mathrm{init}}$ under this factorization, we study them jointly as diversity-control axes in Sec.~\ref{sec:exp:controllability}.

\paragraph{Expressivity hierarchy.}
LoMETab is not identical to BatchEnsemble at $r=1$ because of the identity-residual parameterization. However, for $r \ge 2$, any BatchEnsemble mask $s_kr_k^\top$ can be represented as $1 + A_kB_k\top$, since $s_kr_k^\top - 1$ has rank at most 2.
Thus, the rank axis strictly enlarges the layer-wise family of representable effective weights.


\begin{proposition}[Layer-wise Expressivity Hierarchy]
\label{prop:expressivity}
Fix a layer with $d_{\text{out}}, d_{\text{in}} \ge 2$ and an ensemble size $K \ge 2$. Define the following hypothesis classes of effective weights $(W_1, \ldots, W_K)$:
    \begin{align*}
        \mathcal{H}_{\text{BE}} &= \big\{(W_k)_{k=1}^K : W_k = W \odot (s_k r_k^\top) \text{ for some } W, s_k, r_k \big\}, \\
        \mathcal{H}_{\text{LoMETab}}^{(r)} &= \big\{(W_k)_{k=1}^K : W_k = W \odot (\mathbf{1} + A_k B_k^\top) \text{ for some } W, A_k, B_k \big\}.
    \end{align*}
For rank $r \ge 2$, we have
$$\mathcal{H}_{\text{BE}} \subsetneq \mathcal{H}_{\text{LoMETab}}^{(r)}.$$
\end{proposition}


The proof is provided in App.~\ref{app:proofs}.


\begin{remark}
Proposition~\ref{prop:expressivity} confirms that LoMETab strictly enlarges BatchEnsemble's hypothesis class for $r \ge 2$. Like other implicit ensembles, LoMETab does not aim to match deep ensemble expressivity; our focus is on \emph{controllability} along the rank and initialization axes.
\end{remark}

\subsection{Training Objective}
\label{sec:method:objective}

Each LoMETab member receives a direct learning signal through a member-wise loss, following TabM~\citep{gorishniy2025tabm}. Let $o_k(x)$ denote the raw output of member $k$ after the prediction head. We optimize
\vspace{-0.3em}
\begin{equation}
    \mathcal{L}(x,y)
    =
    \frac{1}{K}
    \sum_{k=1}^{K}
    \ell_{\mathrm{task}}(o_k(x), y),
\label{eqn:method:training_objective}
\end{equation}
\vspace{-0.5em}

where $\ell_{\mathrm{task}}$ is chosen according to the task type: mean squared error for regression, binary cross-entropy for binary classification, and cross-entropy for multiclass classification.

At inference time, the final prediction is obtained by averaging member predictions. For regression, we average the scalar member outputs. For classification, we first convert member outputs into predictive probabilities using sigmoid or softmax, and then average the resulting probabilities. This objective supervises each member-specific adapter directly, rather than only supervising the averaged ensemble output,
allowing the parallel submodels to specialize despite sharing the backbone.
\section{Experiments}
\label{sec:experiments}

\subsection{Experimental Setup}
\label{sec:exp:setup}

\paragraph{Datasets.}
For benchmark performance, we use 37 source datasets from the academic TabM benchmark: 9 \textsc{Default} datasets after excluding \texttt{Housing} (See App.~\ref{app:env_check}), and 28 \textsc{TabZilla} source datasets. The TabZilla sources are further expanded into 68 split-level evaluation tasks following the benchmark protocol; together with the 9 Default datasets, this yields 77 evaluation tasks for per-task tuning analyses. We report benchmark performance at the 37-source-dataset level and hyperparameter-selection statistics over the 77 evaluation tasks.

\textbf{Baselines.} We compare LoMETab against the set of methods evaluated by \citet{gorishniy2025tabm}, spanning gradient-boosted decision trees, attention- and retrieval-based deep tabular models, and MLP-based architectures including the recent implicit-ensemble baseline TabM~\citep{gorishniy2025tabm}. Variants augmented with piecewise-linear embeddings~\citep{gorishniy2022embeddings} are marked with $\dagger$, following the same convention as in TabM, while $\ddagger$ denotes the use of various periodic embeddings. To ensure a fair comparison under identical preprocessing and tuning protocols, all datasets and baseline results are sourced directly from~\citep{gorishniy2025tabm}; full hyperparameter ranges and Optuna budgets are reported in App.~\ref{app:hyperparameters}.


\textbf{Implementation Details.}
The shared weight $W$ uses Kaiming uniform initialization~\citep{he2015delving}.
Adapter initialization follows Sec.~\ref{sec:method:lometab}: both $A_k$ and $B_k$ are initialized from $\mathcal{N}(0,\sigma_{\mathrm{init}}^2)$.
Hyperparameters---
including ensemble size $K$, adapter rank $r$, initialization scale $\sigma_{\mathrm{init}}$, backbone depth $L$, hidden width $d$, learning rate, weight decay, and dropout
---are tuned independently per dataset with up to 100 Optuna trials, and each selected configuration is evaluated across 15 random seeds.
We use RMSE for regression, and Accuracy for classification.
The full search space is provided in App.~\ref{app:hyperparameters}.

\subsection{Main Results}
\label{sec:exp:main}


\begin{figure}[t]
\centering

    \begin{subfigure}{0.32\textwidth}
        \centering
        \includegraphics[width=\linewidth]{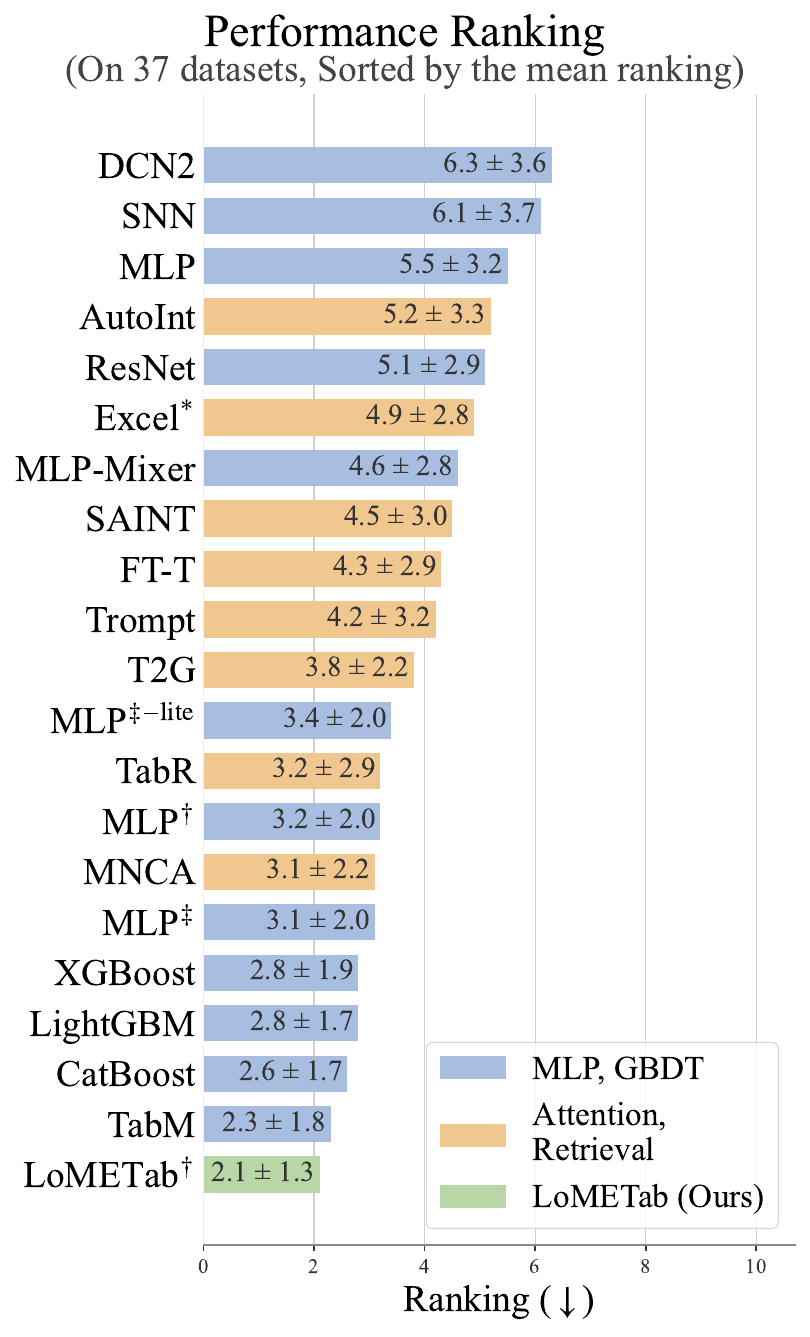}
        \caption{Performance Ranking}
        \label{fig:main_results:1-all}
    \end{subfigure}
    \hfill
    \begin{subfigure}{0.32\textwidth}
        \centering
        \includegraphics[width=\linewidth, trim=0cm 0cm 1cm 0cm]{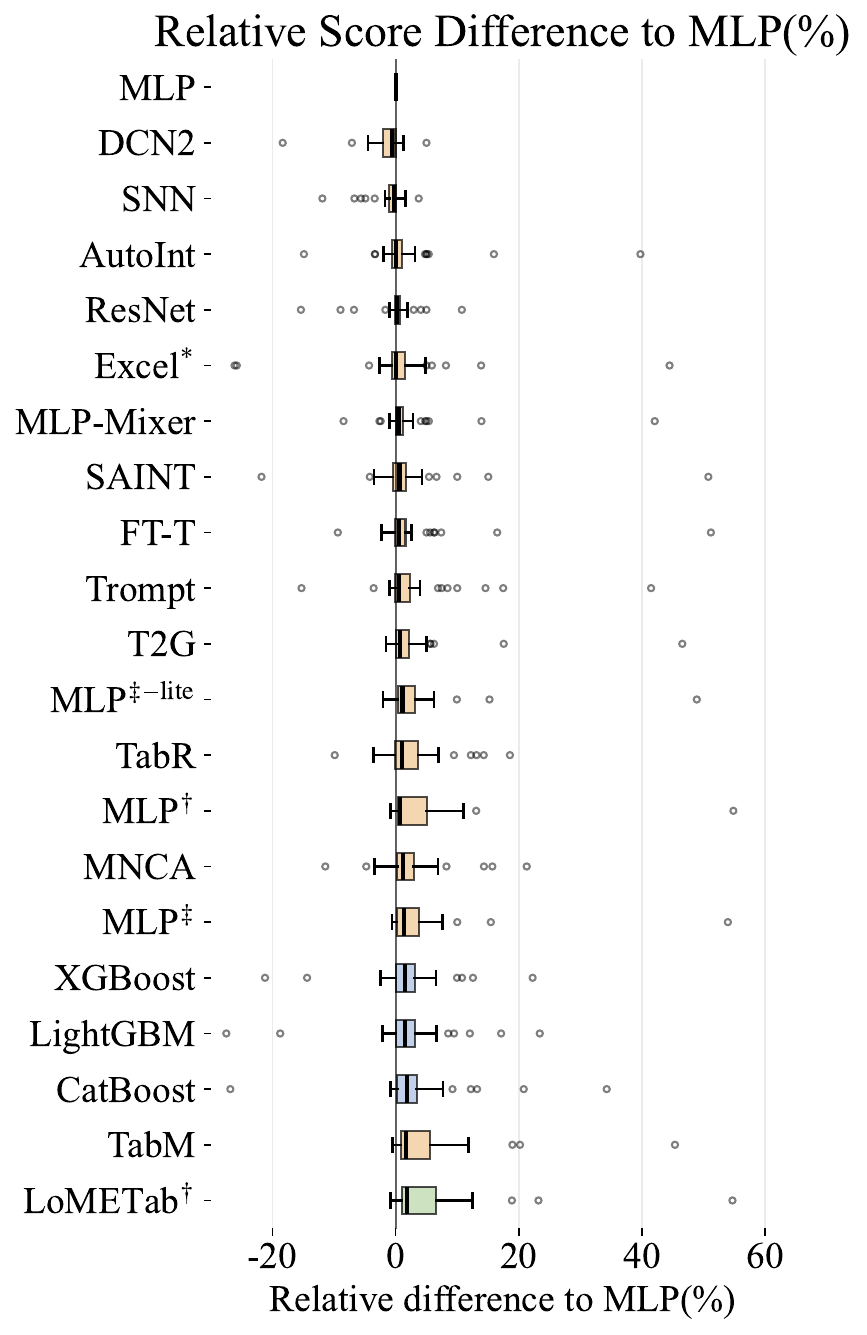}
        \caption{Relative Score to MLP (\%)}
        \label{fig:main_results:2-relative_mlp}
    \end{subfigure}
    \hfill
    \begin{subfigure}{0.32\textwidth}
        \centering
        \includegraphics[width=\linewidth]{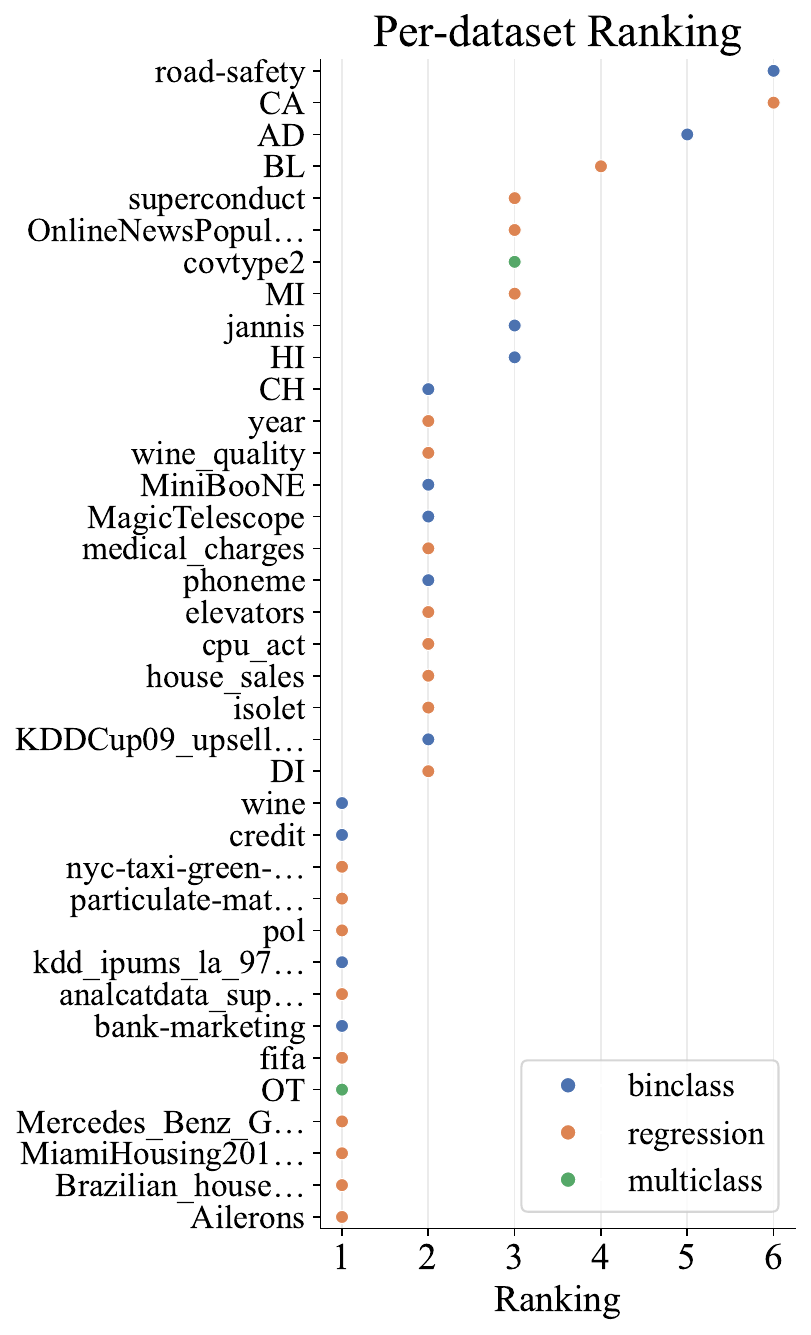}
        \caption{Distribution of Rankings} 
        \label{fig:main_results:3-dist_index}
    \end{subfigure}
    \caption{%
      \textbf{Benchmark performance across 37 academic datasets.}
      (a) Mean and standard deviation of per-dataset ranks; lower is better.
      (b) Sign-corrected relative score difference with respect to MLP (\%).
      (c) Per-dataset ranking of LoMETab.
    }
\label{fig:main_results}

\vspace{-2em}
\end{figure}

Fig.~\ref{fig:main_results} summarizes benchmark performance over the 37 source datasets.
LoMETab falls within the leading cluster of tabular methods, together with strong GBDT baselines, TabM, and recent neural tabular models.
Under the adopted ranking protocol, LoMETab obtains the lowest mean rank among the compared methods: $2.1\pm1.3$, compared with $2.3\pm1.8$ for TabM (Fig.~\ref{fig:main_results:1-all}).
LoMETab also ranks ahead of strong GBDT baselines in mean rank.

The remaining panels provide a dataset-level view of this behavior.
Fig.~\ref{fig:main_results:2-relative_mlp} shows that LoMETab generally improves over or remains close to the MLP reference under the sign-corrected relative metric, placing it within the top-performing group.
Fig.~\ref{fig:main_results:3-dist_index} shows that LoMETab achieves first place on 14 datasets and ranks within the top three on 89.2\% of datasets.
These results establish that LoMETab is competitive across heterogeneous tabular tasks.

We emphasize that these ranks are descriptive rather than evidence of statistical dominance.
As shown in Fig.~\ref{fig:main_results} and observed in recent tabular benchmarks~\citep{grinsztajn2022tree, mcelfresh2024neural}, leading methods form a tight performance band, with small differences depending on the dataset and tuning protocol.
Therefore, we use the benchmark primarily to establish that LoMETab is competitive, and focus the following analyses on the main question of this work: how the LoMETab design axes---rank $r$, initialization scale $\sigma_{\mathrm{init}}$, and ensemble size $K$---affect diversity and performance.
Full per-dataset benchmark scores with standard deviations are reported in App.~\ref{app:benchmark_full_results}.

\subsection{Analyses: Controlling Predictive Diversity}
\label{sec:exp:controllability}



Proposition~\ref{prop:expressivity} shows that LoMETab enlarges the rank-1 hypothesis class, but expressivity alone does not imply that the added capacity manifests as diversity among trained ensemble members.
We proceed in two steps: first, we compare the multiplicative low-rank parameterization underlying LoMETab with an additive low-rank ablation; second, we test whether $(r,\sigma_{\mathrm{init}})$ provide effective control over ensemble diversity.

\paragraph{Diversity metrics.}\label{sec:exp:div_metrics}
For classification datasets, we measure (i) probabilistic diversity via symmetric pairwise Kullback--Leibler divergence between members' predictive distributions~\citep{turkoglu2022filmensemble}, hereafter referred to as \emph{pairwise KL}, and (ii) decision-level diversity via pairwise argmax disagreement~\citep{nam2021diversity}---the fraction of test samples on which two members predict different classes. For regression datasets, where members produce scalar predictions rather than predictive distributions, we use the Krogh--Vedelsby ambiguity~\citep{krogh1994neural}, the average squared deviation of member predictions from the ensemble mean. We additionally report normalized ambiguity $\tilde{A} = A / \sigma_y^2$ to enable comparison across regression tasks with different target scales. Full definitions are given in App.~\ref{app:diversity_metrics}.

\subsubsection{Multiplicative parameterization sustains diversity.}
\label{sec:exp:additive_vs_mul}

\begin{figure}[t]
\centering
\includegraphics[width=\textwidth]{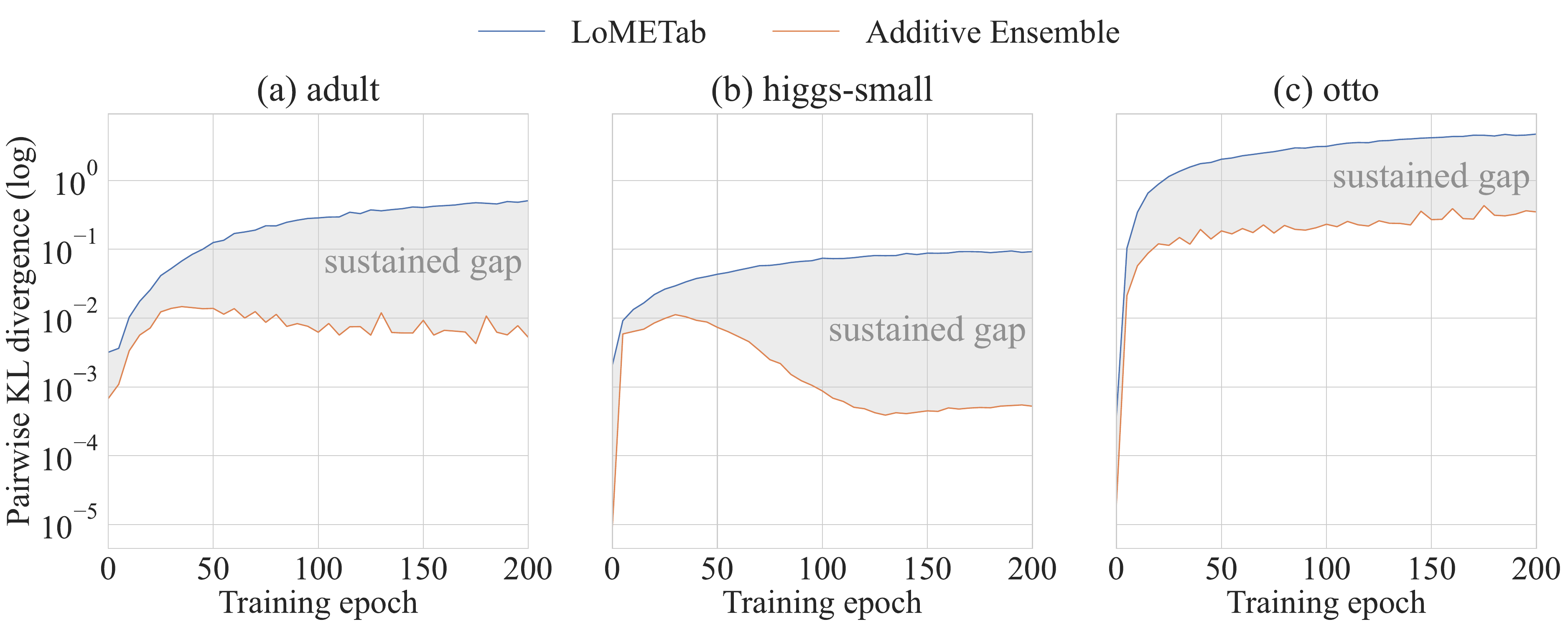}
\caption{%
  \textbf{Sustained diversity gap between additive ensemble and LoMETab.}
  Pairwise KL divergence between ensemble members (log scale) over 200 training epochs on three classification datasets:
  (a)~\texttt{adult}, (b)~\texttt{higgs-small}, and (c)~\texttt{otto}.
  LoMETab maintains substantially higher diversity than the additive ensemble throughout training (grey band).
}
\label{fig:additive_vs_multiplicative}
\end{figure}


The LoMETab layer in Eq.~\eqref{eq:multiplicative_lora_method} places the member-specific low-rank residual inside a multiplicative perturbation of the shared weight.
A natural ablation is an additive low-rank ensemble, $W_k = W + A_kB_k^\top$, which uses the same adapter rank and ensemble size but does not modulate the shared weight through a Hadamard product.
This comparison is not intended to exhaust all possible architectural alternatives; rather, it isolates whether the multiplicative placement of the same low-rank residual helps maintain member diversity during scratch training.

Under identical backbone, optimizer, $K$, and $r$, Fig.~\ref{fig:additive_vs_multiplicative} shows that the additive variant rapidly reaches a low-diversity plateau, whereas LoMETab maintains substantially higher pairwise KL throughout. We refer to this persistent separation as a \emph{sustained gap}.
Because both variants share the same hyperparameters, the gap cannot be attributed to more members or a larger low-rank adapter.
Rather, the result suggests that the placement of the low-rank factors matters for maintaining member diversity during scratch training.
This is consistent with the relative-perturbation form of Eq.~\eqref{eq:multiplicative_lora_method}: LoMETab applies the the low-rank residual as an identity-residual multiplicative factor, $W\odot(1+A_kB_k^\top)$, whereas the additive ablation applies an absolute update $A_kB_k^\top$.
We next study how the LoMETab axes $r$ and $\sigma_{\mathrm{init}}$ control diversity and performance.

\subsubsection{Rank and initialization jointly control predictive diversity.}
\label{sec:exp:controllability_grid}

\begin{figure}[t]
\centering
\includegraphics[width=\textwidth]{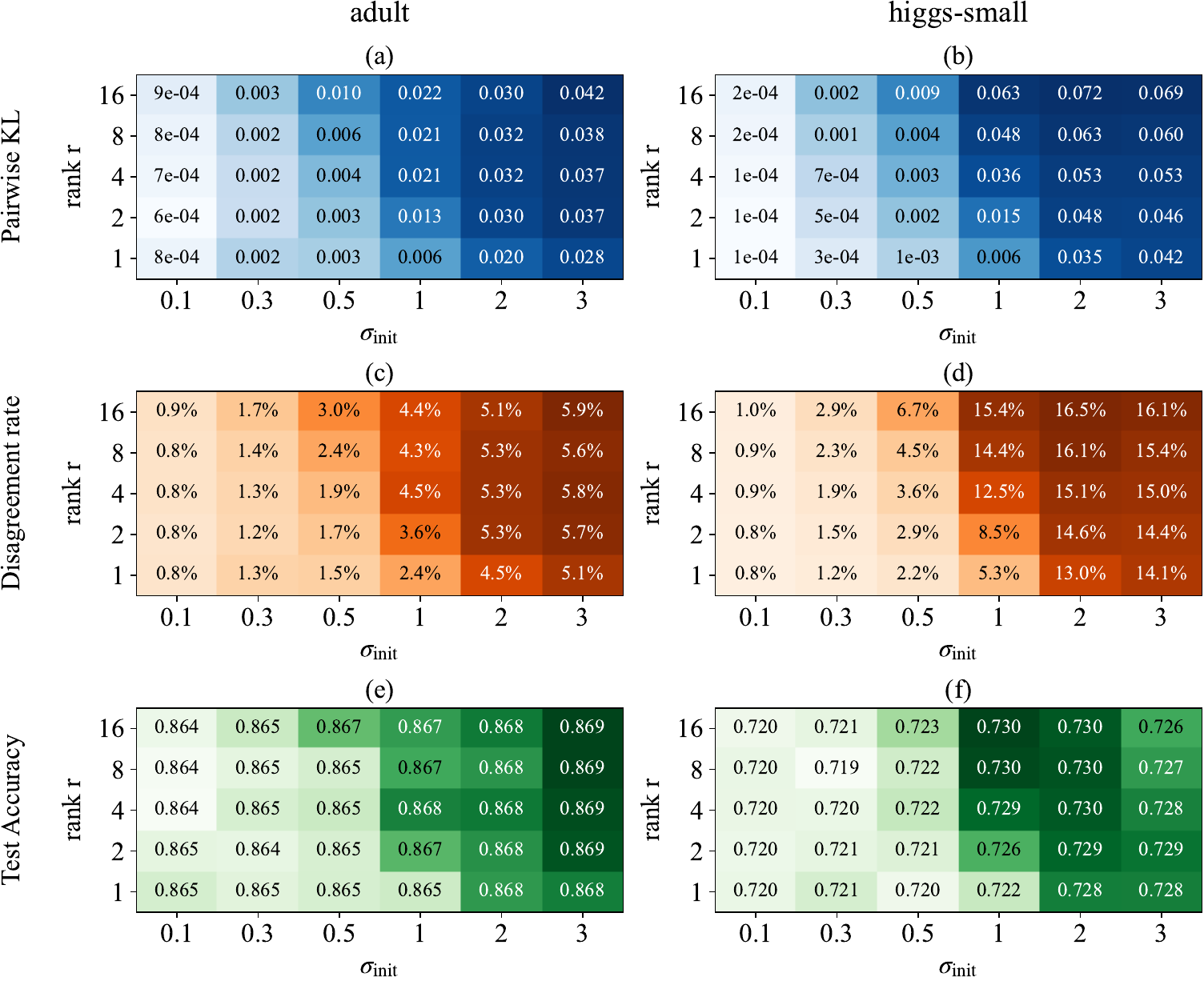}
\caption{
  Diversity control on the $(r,\sigma_{\mathrm{init}})$ grid (Classification).
  Heatmaps show pairwise KL divergence (a)--(b), pairwise argmax disagreement (c)--(d), and test accuracy (e)--(f) across rank $r$ and initialization scale $\sigma_{\mathrm{init}}$.
  Values are averaged over 15 random seeds; standard deviations are omitted for readability (See App.~\ref{app:diversity_grid}).
}
\label{fig:ctrl_grid}
\vspace{-1em}
\end{figure}

\paragraph{Grid study setup.}
We sweep $r \in \{1,2,4,8,16\}$ and $\sigma_{\mathrm{init}} \in \{0.1,0.3,0.5,1,2,3\}$, yielding 30
configurations per dataset. All non-swept hyperparameters are fixed to the Optuna-selected configuration for that dataset, and the ensemble size is fixed at $K=32$.
We focus on two classification datasets (\texttt{adult}, \texttt{higgs-small}) and two regression datasets (\texttt{california}, \texttt{black-friday}); full numerical grid results with seed-level standard deviations are analyzed in App.~\ref{app:diversity_grid}.
We interpret cells whose diversity metrics remain near the empirical floor of the grid as low-diversity regimes.

\paragraph{Classification diversity.}
Fig.~\ref{fig:ctrl_grid} shows that $r$ and $\sigma_{\mathrm{init}}$ jointly control probabilistic diversity (Fig.~\ref{fig:ctrl_grid} (a-b)).
At $\sigma_{\mathrm{init}}=0.1$, rank has little effect: pairwise KL remains near the low-diversity floor on both \texttt{adult} ($6\times10^{-4}$--$9\times10^{-4}$) and \texttt{higgs-small} ($1\times10^{-4}$--$2\times10^{-4}$).
As $\sigma_{\mathrm{init}}$ increases, the rank axis becomes active. At $\sigma_{\mathrm{init}}=1.0$, the rank sweep increases KL from about $0.006$ to $0.022$ on \texttt{adult}, and from about $0.006$ to $0.063$ on \texttt{higgs-small}.
Thus, member diversity depends on the joint choice of $r$ and $\sigma_{\mathrm{init}}$:
when $\sigma_{\mathrm{init}}$ is small,
the diversity metrics remain near their empirical floor regardless of rank,
whereas larger $\sigma_{\mathrm{init}}$ makes the rank axis effective.

This probabilistic diversity is also reflected in final decisions (Fig.~\ref{fig:ctrl_grid} (c)--(d)).
At the smallest initialization scale ($\sigma_{\mathrm{init}}=0.1$), pairwise argmax disagreement rate is around or below $1\%$ on both datasets.
At larger scales, disagreement reaches about $5$--$6\%$ on \texttt{adult} and about $14$--$17\%$ on \texttt{higgs-small}.
Hence, $(r,\sigma_{\mathrm{init}})$ controls not only probability-level variation but also decision-level member variation.

\begin{figure}[t]
\centering
\includegraphics[width=1.0\linewidth]{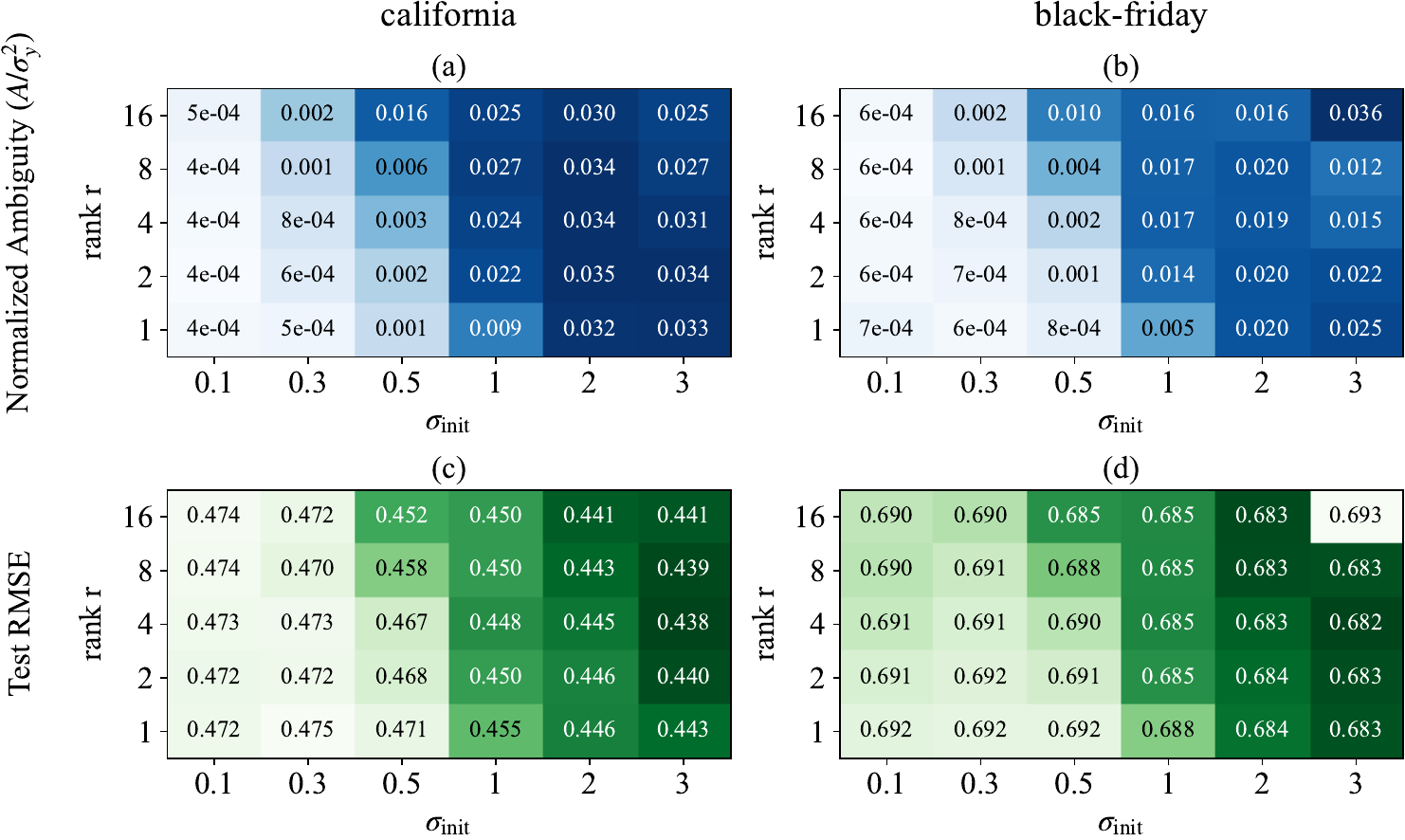}
\caption{
  Diversity control on the $(r,\sigma_{\mathrm{init}})$ grid (Regression).
  Heatmaps show normalized ambiguity $\tilde{A}=A/\sigma_y^2$ and test RMSE across rank $r$ and initialization scale $\sigma_{\mathrm{init}}$.
  Values are averaged over 15 random seeds; standard deviations are omitted for readability (See App.~\ref{app:diversity_grid}).
}
\label{fig:ctrl_grid_reg}
\end{figure}

\paragraph{Regression diversity.}
Fig.~\ref{fig:ctrl_grid_reg} shows that regression follows a similar rank--scale activation pattern (Fig.~\ref{fig:ctrl_grid_reg}(a, b)).
At $\sigma_{\mathrm{init}}=0.1$, normalized ambiguity remains near the metric floor for all ranks on both \texttt{california} and \texttt{black-friday}, so increasing rank alone has little effect.
As $\sigma_{\mathrm{init}}$ increases, the rank axis becomes active: at $\sigma_{\mathrm{init}}=0.5$, ambiguity increases from roughly $0.001$ to $0.016$ on \texttt{california} and from about $8\times10^{-4}$ to $0.010$ on \texttt{black-friday}.

Overall, regression ambiguity tends to increase with $\sigma_{\mathrm{init}}$, with mild saturation or non-monotonicity across rank at larger scales.
Thus, $(r,\sigma_{\mathrm{init}})$ also controls output-level member variation for regression,
but the largest-ambiguity cells do not necessarily coincide with the best RMSE.

\paragraph{Diversity--performance alignment.}
Across Fig.~\ref{fig:ctrl_grid} and Fig.~\ref{fig:ctrl_grid_reg}, predictive performance broadly improves as the grid moves away from the low-$\sigma_{\mathrm{init}}$ and low-diversity regime: classification accuracy increases in higher-diversity regions (Fig.~\ref{fig:ctrl_grid} (e)--(f)),
and regression RMSE generally decreases as $\sigma_{\mathrm{init}}$ grows (Fig.~\ref{fig:ctrl_grid_reg} (c)--(d)).
The trend is not strictly monotone in every cell, especially across rank and at extreme scales, but the results indicate that $(r,\sigma_{\mathrm{init}})$ are useful design handles for both diversity control and dataset-dependent performance tuning.

\subsection{Ablation Studies}
\label{sec:exp:design_axes}

Sec.~\ref{sec:exp:controllability_grid} shows axis $r$ and $\sigma_{\mathrm{init}}$ jointly activate
predictive diversity. We next ask whether these axes are useful for performance
and calibration tuning under a controlled protocol.
We conduct two single-axis sweeps on the Default benchmark datasets, comparing LoMETab against TabM $K$-scaling:
an $r$-sweep ($r \in \{1,2,4,8,16\}$) with $\sigma_{\mathrm{init}}=1.0$ fixed,
and a $\sigma_{\mathrm{init}}$-sweep ($\sigma_{\mathrm{init}} \in \{0.1,0.3,0.5,1,2\}$) with $r=16$ fixed.
All non-swept hyperparameters are shared within each comparison, and all results are reported as sign-corrected $\Delta\%$ relative to the TabM $K=32$ baseline. Calibration is measured via ECE (App.~\ref{app:diversity_metrics}). Full results for all sweep conditions are reported in App.~\ref{app:per-dataset-ablations}.

\input{tables/tab_axis_impact.tex}

\paragraph{Both $r$ and $\sigma_{\mathrm{init}}$ act as performance axes.}

Tab.~\ref{tab:axis_impact} contrasts four strategies against the (a) TabM $K{=}32$ baseline.
First, changing $K$ alone is not consistently beneficial under this protocol.
(b) Reducing TabM to $K=16$ yields broadly comparable results, whereas (c) increasing to $K=64$ degrades performance on most of the datasets, which is consistent with prior observations that excessive $K$ may be detrimental~\citep{gorishniy2025tabm}.
In contrast, (d)--(e) LoMETab sweeps at $K{=}16$ surpass or match the TabM $K=32$ baseline on a majority of datasets by tuning either $r$ or $\sigma_{\mathrm{init}}$ alone, suggesting that $r$ and $\sigma_{\mathrm{init}}$ provide useful performance
tuning axes beyond changing $K$ alone.
The strongest gains appear on CO and CH.
The selected values of $r$ and $\sigma_{\mathrm{init}}$ vary across datasets, supporting our view that they should be treated as tunable design axes rather than maximized by default.
Crucially, (d)--(e) often share the same gain via different selected values---for instance, CO improves by $+2.6\%$ through $r{=}1$ and equally through $\sigma_{\mathrm{init}}{=}0.3$---suggesting that the two axes provide complementary, dataset-dependent tuning effects.
Where they diverge---as in CA, where the $r$-sweep yields $-0.8\%$ while the $\sigma_{\mathrm{init}}$-sweep achieves $+0.1\%$
---the result suggests that $r$ and $\sigma_{\mathrm{init}}$ capture distinct effects, motivating joint tuning in practice.

\input{tables/tab_pareto.tex}


\paragraph{Accuracy--Calibration trade-offs.}


Tab.~\ref{tab:pareto} shows that the same distinction matters for calibration.
(b)--(c) Increasing $K$ can improve ECE on some datasets, but often at the cost of accuracy; on CO and CH, both accuracy and ECE degrade.
(d)--(e) LoMETab provides additional ways to tune this trade-off through $r$ and $\sigma_{\mathrm{init}}$:
some configurations improve both accuracy and ECE on HI, OT, and CH, while CO improves substantially in accuracy (+2.6\%) but worsens in ECE. This indicates that calibration is not solved by
more diversity alone; the form and scale of the member deviations matter.
Together, these ablations suggest that useful diversity in implicit ensembles is not just a matter of averaging more members; it also depends on how members are allowed to deviate from the shared backbone: TabM changes the number of rank-1 members through $K$, whereas LoMETab changes the rank and scale of the member-specific perturbations through $(r,\sigma_{\mathrm{init}})$.

\vspace{-1em}

\section{Conclusion}
\label{sec:conclusion}


We introduced LoMETab, a rank-$r$ multiplicative implicit ensemble that replaces the fixed rank-1 BatchEnsemble/TabM mask with an identity-residual low-rank Hadamard perturbation.
This yields a strictly larger layer-wise family for $r\ge2$, while exposing $r$ and $\sigma_{\mathrm{init}}$ as practical axes for controlling member diversity.
Empirically, LoMETab remains competitive on tabular benchmarks, sustains higher pairwise KL than an additive low-rank ablation, and lets pairwise KL vary by more than two orders of magnitude, with corresponding effects on argmax disagreement, regression ambiguity, and dataset-dependent performance.
These results suggest that, for tabular implicit ensembles, how members are allowed to deviate from a shared backbone is a central design question.

\paragraph{Limitations and Future Work.}
Our analysis establishes that $(r, \sigma_{\text{init}})$ provide two-axis control over ensemble diversity, but leaves open the question of what downstream benefits this controllability enables. Whether the diversity exposed by our formulation translates into improved uncertainty estimation or out-of-distribution detection on tabular data---tasks of increasing interest in the tabular deep learning community---is the most natural direction for future work.

\bibliographystyle{plainnat}
\bibliography{references}


\newpage
\appendix
\section{Proof of Proposition~\ref{prop:expressivity}}
\label{app:proofs}

\paragraph{Notation.} We write $\odot$ for the Hadamard (element-wise) product, $\oslash$ for element-wise division, and $\mathbf{1}$ for the all-ones matrix (dimensions inferred from context). For brevity we write $m := d_{\text{out}}$ and $n := d_{\text{in}}$. The hypothesis classes are:
\begin{align*}
\mathcal{H}_{\text{BE}} &= \left\{ \left(W \odot (s_k r_k^\top)\right)_{k=1}^K : W \in \mathbb{R}^{m \times n},\; r_k \in \mathbb{R}^n,\; s_k \in \mathbb{R}^m \right\}, \\
\mathcal{H}_{\text{LoMETab}}^{(r)} &= \left\{ \left(W \odot (\mathbf{1} + A_k B_k^\top)\right)_{k=1}^K : W \in \mathbb{R}^{m \times n},\; A_k \in \mathbb{R}^{m \times r},\; B_k \in \mathbb{R}^{n \times r} \right\}.
\end{align*}
All shared weights $W$ are assumed to have non-zero entries; this is a generic condition satisfied on a full-measure subset of $\mathbb{R}^{m \times n}$.

\subsection{Part (a): Inclusion $\mathcal{H}_{\text{BE}} \subseteq \mathcal{H}_{\text{LoMETab}}^{(r)}$}
\label{app:proof:ia}

We first state three lemmas. Lemmas~\ref{lem:rank_outer} and~\ref{lem:rank_subadd} are standard (see, e.g., Horn \& Johnson~\cite{horn2012matrix}); We prove Lemmas~\ref{lem:rank_factorization} and~\ref{lem:be_ratio_rank1}, which are the key constructions.

\begin{lemma}[Rank of outer product]
\label{lem:rank_outer}
For any $u \in \mathbb{R}^m$ and $v \in \mathbb{R}^n$, $\mathrm{rank}(u v^\top) \le 1$.
\end{lemma}

\begin{lemma}[Rank subadditivity]
\label{lem:rank_subadd}
For any $X, Y \in \mathbb{R}^{m \times n}$, $\mathrm{rank}(X + Y) \le \mathrm{rank}(X) + \mathrm{rank}(Y)$.
\end{lemma}

\begin{lemma}[Rank-$\rho$ factorization with padding]
\label{lem:rank_factorization}
Let $r \ge \rho \ge 0$ and $M \in \mathbb{R}^{m \times n}$ with $\mathrm{rank}(M) = \rho$. Then there exist $A \in \mathbb{R}^{m \times r}$ and $B \in \mathbb{R}^{n \times r}$ such that $A B^\top = M$.
\end{lemma}

\begin{proof}
We proceed by case analysis on $\rho$.

\textbf{Case $\rho = 0$.} Then $M = \mathbf{0}_{m \times n}$. Setting $A := \mathbf{0}_{m \times r}$ and $B := \mathbf{0}_{n \times r}$ gives
\[
A B^\top = \mathbf{0}_{m \times r} \cdot \mathbf{0}_{r \times n} = \mathbf{0}_{m \times n} = M.
\]

\textbf{Case $\rho \ge 1$.} By the compact SVD theorem (e.g., Horn \& Johnson~\cite{horn2012matrix}), there exist $U \in \mathbb{R}^{m \times \rho}$ with $U^\top U = I_\rho$, $V \in \mathbb{R}^{n \times \rho}$ with $V^\top V = I_\rho$, and $\Sigma = \mathrm{diag}(\sigma_1, \ldots, \sigma_\rho) \in \mathbb{R}^{\rho \times \rho}$ with $\sigma_i > 0$, such that
\[
M = U \Sigma V^\top.
\]
Absorbing $\Sigma$ into $U$, set $\tilde{U} := U \Sigma \in \mathbb{R}^{m \times \rho}$, so that $M = \tilde{U} V^\top$. Since $r - \rho \ge 0$, we pad with zero columns:
\[
A := \begin{bmatrix} \tilde{U} & \mathbf{0}_{m \times (r-\rho)} \end{bmatrix} \in \mathbb{R}^{m \times r}, \qquad
B := \begin{bmatrix} V & \mathbf{0}_{n \times (r-\rho)} \end{bmatrix} \in \mathbb{R}^{n \times r}.
\]
Writing $B^\top = \begin{bmatrix} V^\top \\ \mathbf{0}_{(r-\rho) \times n} \end{bmatrix}$ and applying block matrix multiplication,
\[
A B^\top = \tilde{U} V^\top + \mathbf{0}_{m \times (r-\rho)} \cdot \mathbf{0}_{(r-\rho) \times n} = \tilde{U} V^\top = M. \qedhere
\]
\end{proof}



\begin{lemma}
\label{lem:be_ratio_rank1}
For any $(W_1, \ldots, W_K) \in \mathcal{H}_{\text{BE}}$ with shared weight $W$ and scaling vectors $s_k, r_k$ all having nonzero entries, and for any two members $i \ne j$, the element-wise ratio $Q_{ij} := W_i \oslash W_j$ satisfies $\mathrm{rank}(Q_{ij}) \le 1$.
\end{lemma}
\begin{proof}
By the definition of BE, $W_i = W \odot (s_i r_i^\top)$ and $W_j = W \odot (s_j r_j^\top)$. Since $W$, $s_j$, and $r_j$ all have nonzero entries, $Q_{ij}$ is well-defined. Taking the element-wise ratio:
\[
(Q_{ij})_{ab} = \frac{W_{ab}\, (s_i)_a (r_i)_b}{W_{ab}\, (s_j)_a (r_j)_b} = \frac{(s_i)_a}{(s_j)_a} \cdot \frac{(r_i)_b}{(r_j)_b}.
\]
Defining $u \in \mathbb{R}^m$ with $u_a = (s_i)_a / (s_j)_a$ and $v \in \mathbb{R}^n$ with $v_b = (r_i)_b / (r_j)_b$, we have $Q_{ij} = u v^\top$, an outer product of rank at most 1.
\end{proof}

\begin{proof}[Main proof of Part (a)]
Let $(W_1^{\text{BE}}, \ldots, W_K^{\text{BE}}) \in \mathcal{H}_{\text{BE}}$. By definition, there exist $W \in \mathbb{R}^{m \times n}$, $s_k \in \mathbb{R}^m$, and $r_k \in \mathbb{R}^n$ such that
\[
W_k^{\text{BE}} = W \odot (s_k r_k^\top), \quad k = 1, \ldots, K.
\]
We construct $A_k, B_k$ such that $W \odot (\mathbf{1} + A_k B_k^\top) = W_k^{\mathrm{BE}}$ for each $k = 1, \ldots, K$.

It suffices to find $A_k$ and $B_k$ satisfying
\begin{equation}
\label{eq:mask_identity_ia}
A_k B_k^\top = s_k r_k^\top - \mathbf{1},
\end{equation}
since then $W \odot (\mathbf{1} + A_k B_k^\top) = W \odot (s_k r_k^\top) = W_k^{\text{BE}}$.
Let $M_k := s_k r_k^\top - \mathbf{1}$. Since $\mathbf{1} = \mathbf{1}_m \mathbf{1}_n^\top$, 
Lemmas~\ref{lem:rank_outer} and~\ref{lem:rank_subadd} give
\[
\rho_k := \mathrm{rank}(M_k) \le \mathrm{rank}(s_k r_k^\top) + \mathrm{rank}(\mathbf{1}) \le 2.
\]
Since $r \ge 2 \ge \rho_k$, Lemma~\ref{lem:rank_factorization} yields $A_k \in \mathbb{R}^{m \times r}$ 
and $B_k \in \mathbb{R}^{n \times r}$ with $A_k B_k^\top = M_k$, 
establishing~\eqref{eq:mask_identity_ia}. Since the choice of BE element was arbitrary,
$\mathcal{H}_{\text{BE}} \subseteq \mathcal{H}_{\text{LoMETab}}^{(r)}$.





\subsection{Part (b): Strictness $\mathcal{H}_{\text{LoMETab}}^{(r)} \not\subseteq \mathcal{H}_{\text{BE}}$}

We exhibit an element of $\mathcal{H}_{\text{LoMETab}}^{(r)}$ that cannot be represented in $\mathcal{H}_{\text{BE}}$.


\paragraph{Construction of counterexample.}
Let $W \in \mathbb{R}^{m \times n}$ be any matrix with all entries nonzero, and let $e_1, e_2 \in \mathbb{R}^m$ and $f_1, f_2 \in \mathbb{R}^n$ denote the first two standard basis vectors. Define
\[
M_1 := e_1 f_1^\top + e_2 f_2^\top, \qquad M_2 := \mathbf{0}.
\]
Since $M_1$ has exactly two nonzero rows, namely $f_1^\top$ and $f_2^\top$, which are linearly independent, $\mathrm{rank}(M_1) = 2$. By Lemma~\ref{lem:rank_factorization}, there exist $A_1 \in \mathbb{R}^{m \times r}$ and $B_1 \in \mathbb{R}^{n \times r}$ with $A_1 B_1^\top = M_1$. Set $A_2 = \mathbf{0}$, so that $A_2 B_2^\top = M_2 = \mathbf{0}$. Define
\[
W_1 := W \odot (\mathbf{1} + M_1), \qquad W_2 := W \odot (\mathbf{1} + M_2) = W.
\]
Then $(W_1, W_2) \in \mathcal{H}_{\mathrm{LoMETab}}^{(r)}$ by construction. We now show $(W_1, W_2) \notin \mathcal{H}_{\mathrm{BE}}$.
\[
W_1 = W \odot (\mathbf{1} + A_1 B_1^\top), \qquad W_2 = W \odot \mathbf{1} = W,
\]
and the element-wise ratio $Q_{12} := W_1 \oslash W_2$ satisfies
\[
(Q_{12})_{ab} = \begin{cases} 2 & (a,b) \in \{(1,1),(2,2)\}, \\ 1 & \text{otherwise.} \end{cases}
\]
The top-left $2 \times 2$ submatrix of $Q_{12}$ is $\begin{pmatrix} 2 & 1 \\ 1 & 2 \end{pmatrix}$ with determinant $3 \ne 0$, hence $\mathrm{rank}(Q_{12}) \ge 2$. By the contrapositive of Lemma~\ref{lem:be_ratio_rank1}, no BE parameterization can produce the pair $(W_1, W_2)$. Therefore $\mathcal{H}_{\text{BE}} \subsetneq \mathcal{H}_{\text{LoMETab}}^{(r)}$. \end{proof}

\section{Benchmark Protocol, Baseline Sourcing and Dataset Exclusion}
\label{app:env_check}


This appendix details the benchmark protocol used in Sec.~\ref{sec:experiments}, including the inherited TabM pipeline, baseline result sourcing, an environment consistency check, the rationale for excluding the house dataset, and the ranking methodology.

\paragraph{Inherited pipeline and baseline sourcing.} Our LoMETab implementation inherits the official TabM codebase~\citep{gorishniy2025tabm}, which itself adopts the experimental framework of~\citet{gorishniy2024modernnca}. Our implementation inherits identical dataset loaders, train/validation/test splits, preprocessing (TypedFeatureEmbeddings for categorical features, TargetScaler for regression targets), metric computation, and training protocol from this codebase. All baseline results are taken from the appendix tables of~\citet{gorishniy2025tabm} and the associated report files released in the official repository, which provides configuration and metric JSON files for all reported methods under a unified evaluation pipeline.
Because LoMETab is implemented within the same forked pipeline, its results use the same data handling, preprocessing, splits, and metric computation as the released TabM benchmark artifacts; the compared baselines are taken from the corresponding published reports.

\paragraph{Environment consistency check.} To verify that our forked codebase preserves the evaluation environment of~\citet{gorishniy2025tabm}, we reproduced the publicly released TabM configuration on the \texttt{brazilian\_houses} dataset in our environment using 15 independent random seeds. The resulting mean test RMSE is consistent with the appendix value reported by~\citet{gorishniy2025tabm}, supporting the claim that LoMETab and baseline results are evaluated under the same pipeline.

\paragraph{Excluded dataset: \texttt{house}.}
During our pipeline verification, we identified a 
discrepancy on the \texttt{house} dataset. Its target distribution has median 
33{,}200 and mean 50{,}074, right-skewed with the bottom 10\% near 14{,}999---values 
on the order of $10^4$. However, RMSE values reported by prior tabular DL methods 
on this dataset cluster tightly around 3.0 across nearly all models, which is 
inconsistent with both the raw target scale and standard z-score normalization. 
The cause of this mismatch is unclear from publicly available materials, and to 
avoid potential pipeline inconsistency we exclude \texttt{house} from our evaluation.

\paragraph{Ranking methodology.} Rankings reported in Sec.~\ref{sec:exp:main} are computed from the published mean scores of each method on each dataset, following standard protocol in recent tabular DL benchmarks~\citep{gorishniy2024modernnca, rubachev2024tabred}. Ties are broken by the alphabetical order of method names. We report rank distributions and paired comparisons as descriptive evidence; we do not claim statistical significance of rank differences, as seed-level raw scores are not available for all baselines.

\section{Diversity and Calibration Metrics: Full Definitions}
\label{app:diversity_metrics}

We provide the full definitions of the diversity metrics used in Sec.~\ref{sec:exp:div_metrics}. Throughout, $K$ denotes the ensemble size, $N$ the number of test samples, $p_i(x_n)$ the predicted class-probability vector of member $i$ on test sample $x_n$, and $f_i(x_n)$ the scalar prediction of member $i$ on $x_n$.

\paragraph{Pairwise KL (classification).}
We measure probabilistic diversity by the symmetric pairwise Kullback--Leibler divergence between members' predictive distributions, averaged over all member pairs and test samples:
\begin{equation}
\mathrm{KL} = \frac{2}{K(K-1)} \sum_{i=1}^{K} \sum_{j=i+1}^{K} \frac{1}{N} \sum_{n=1}^{N} \tfrac{1}{2}\left[\mathrm{KL}(p_i(x_n) \Vert p_j(x_n)) + \mathrm{KL}(p_j(x_n) \Vert p_i(x_n))\right].
\end{equation}
The symmetric form ensures that the metric is invariant under the choice of reference member. Following~\citet{turkoglu2022filmensemble}, we average over all $\binom{K}{2}$ pairs.

\paragraph{Pairwise argmax disagreement (classification).}
We measure decision-level diversity by the fraction of test samples on which two members predict different classes, averaged over all member pairs. Let
\(\hat{y}_i(x_n)=\arg\max_c p_i^{(c)}(x_n)\) denote the class predicted by member \(i\). We define
\begin{equation}
    D =
    \frac{2}{K(K-1)}
    \sum_{i=1}^{K} \sum_{j=i+1}^{K}
    \frac{1}{N}
    \sum_{n=1}^{N}
    \mathbb{1}\!\left[\hat{y}_i(x_n) \ne \hat{y}_j(x_n)\right],
\end{equation}
where \(\mathbb{1}[\cdot]\) is the indicator function. This is the empirical pairwise disagreement rate following~\citet{nam2021diversity}.


\paragraph{Krogh--Vedelsby ambiguity (regression).}
For regression, neither pairwise KL nor argmax disagreement is defined since members produce scalar predictions rather than probability vectors. We instead use the ensemble \emph{ambiguity} term from the Krogh--Vedelsby decomposition~\citep{krogh1994neural}, defined as the average squared deviation of individual member predictions from the ensemble mean:
\begin{equation}
A = \frac{1}{NK} \sum_{n=1}^{N} \sum_{i=1}^{K} \left(f_i(x_n) - \bar{f}(x_n)\right)^2, \quad \bar{f}(x_n) = \frac{1}{K}\sum_{i=1}^{K} f_i(x_n).
\end{equation}
Ambiguity is computed in the original (un-normalized) target scale.

To enable comparison across regression tasks with different target scales, we additionally report \emph{normalized ambiguity}:
\begin{equation}
\tilde{A} = \frac{A}{\sigma_y^2},
\end{equation}
where $\sigma_y^2$ is the variance of the raw training targets.

\paragraph{Krogh--Vedelsby decomposition.}
The ambiguity term arises naturally from the well-known decomposition of ensemble error~\citep{krogh1994neural}:
\begin{equation}
\underbrace{\bigl(\bar{f}(x) - y\bigr)^2}_{\text{ensemble error}} 
= \underbrace{\frac{1}{K}\sum_{k=1}^{K} \bigl(f_k(x) - y\bigr)^2}_{\text{average member error}} 
- \underbrace{\frac{1}{K}\sum_{k=1}^{K} \bigl(f_k(x) - \bar{f}(x)\bigr)^2}_{\text{ambiguity}}.
\end{equation}
Hence ensemble error decreases only when the gain in ambiguity outweighs the increase in average member error---a relationship we leverage when interpreting the regression diversity--accuracy trade-off in §\ref{sec:exp:controllability}.

\paragraph{Expected Calibration Error (classification).}
We measure calibration quality via the Expected Calibration Error (ECE)~\citep{naeini2015obtaining},
which quantifies the discrepancy between predicted confidence and empirical accuracy. For ECE, we
use the ensemble-averaged predictive probability
\[
\bar{p}(x_n)=K^{-1}\sum_{i=1}^{K}p_i(x_n).
\]
The predicted label and confidence are
\[
\hat{y}_n=\arg\max_c \bar{p}^{(c)}(x_n),
\qquad
\mathrm{conf}_n=\max_c \bar{p}^{(c)}(x_n).
\]
We partition test samples into \(M\) equally spaced confidence bins
\(\{B_m\}_{m=1}^{M}\) and compute
\begin{equation}
\mathrm{ECE}
=
\sum_{m=1}^{M}
\frac{|B_m|}{N_{\mathrm{test}}}
\left|
\mathrm{acc}(B_m)-\mathrm{conf}(B_m)
\right|,
\end{equation}
where
\[
\mathrm{acc}(B_m)
=
\frac{1}{|B_m|}
\sum_{n\in B_m}
\mathbb{1}[\hat{y}_n=y_n],
\qquad
\mathrm{conf}(B_m)
=
\frac{1}{|B_m|}
\sum_{n\in B_m}
\mathrm{conf}_n.
\]
Empty bins contribute zero. We use \(M=15\) bins following~\citet{guo2017calibration}.
ECE\(\downarrow\) indicates better calibration.



\section{Checklist Details}

\paragraph{Existing assets and licenses}
\label{app:licenses}
Our implementation builds on the official TabM codebase and benchmark artifacts released by
Gorishniy et al.~\citep{gorishniy2025tabm}, including dataset loaders, preprocessing routines,
train/validation/test splits, metric computation, hyperparameter reports, and baseline result files.
The official TabM repository is released under the Apache-2.0 License. The TabZilla source
datasets used by the benchmark originate from OpenML; for each dataset, we follow the original
OpenML metadata, including dataset identifiers, creators, citations, and license fields. We do not
redistribute raw datasets; instead, datasets are accessed through the original benchmark/OpenML
pipeline. Hyperparameter optimization uses Optuna~\citep{akiba2019optuna}, which is released
under the MIT License. We credit all original papers and asset owners in the relevant sections and
respect the corresponding license terms.

\paragraph{Broader impacts}
\label{app:etc:broader_impacts}
LoMETab is a general-purpose tabular learning method. Potential positive impacts include improving predictive modeling in domains where tabular data is common, such as operations, science, and business analytics. Potential negative impacts may arise if such models are deployed in high-stakes settings such as finance, healthcare, or employment without appropriate validation, fairness analysis, calibration checks, and human oversight. Our work does not directly address deployment safety or fairness, and practitioners should evaluate these aspects before deployment.


\section{Hyperparameter Tuning Details}
\label{app:hyperparameters}

This appendix provides the detailed tuning protocol for LoMETab, complementing the high-level description in Sec.~\ref{sec:experiments}.

\paragraph{Tuning framework.} We use Optuna~\citep{akiba2019optuna} with the MedianPruner for pruning underperforming trials. The tuning objective is the validation metric: RMSE for regression (minimized), and accuracy for classification (binary and multiclass, maximized).

\paragraph{Trial budget.} We use up to 100 Optuna trials per (variant, dataset) combination. Exact per-dataset trial counts, which may be reduced for computationally expensive datasets, are released with our code. Each (variant, dataset) pair is tuned independently.

\paragraph{Search space.}

The hyperparameters tuned by Optuna for the main benchmark evaluation are summarized in Table~\ref{tab:optuna_hparams}. This HPO space differs from the controlled sweep grids used for analysis: Sec.~\ref{sec:exp:controllability_grid} sweeps \(\sigma_{\mathrm{init}}\in\{0.1,0.3,0.5,1,2,3\}\) at fixed \(K=32\), while Sec.~\ref{sec:exp:design_axes} uses single-axis sweeps over \(r\in\{1,2,4,8,16\}\) and \(\sigma_{\mathrm{init}}\in\{0.1,0.3,0.5,1,2\}\).

\begin{table}[h!]
\centering
\caption{
    The hyperparameter tuning space for LoMETab optimized with Optuna.
}
\label{tab:optuna_hparams}
{\renewcommand{\arraystretch}{1.2}
\begin{tabular}{ll}
    \toprule
    Parameter           & Distribution or Value\\
    \midrule
    Ensemble size $K$   & $\{16, 32\}$ \\
    Rank $r$            & $\{1, 2, 4, 8, 16\}$ \\
    Initialization scale $\sigma_{\text{init}}$ & $\{0.1, 0.3, 0.5, 1.0\}$ \\
    Hidden dimension    & $\mathrm{UniformInt}[64,1024]$ (step=16) \\
    \# layers           & $\mathrm{UniformInt}[1,4]$ \\
    Learning rate       & $\mathrm{LogUniform}[1e\text{-}4, 5e\text{-}3]$ \\
    Weight decay        & $\{0, \mathrm{LogUniform}[1e\text{-}4, 1e\text{-}1]\}$ \\
    Dropout             & $\{0, \mathrm{Uniform}[0.0,0.5]\}$ \\
    PLE embedding dim $d_{\text{emb}}$ & $\mathrm{UniformInt}[8,32]$ (step=4) \\
    \# PLE bins         & $\mathrm{UniformInt}[2,128]$ \\
    \bottomrule
\end{tabular}}
\end{table}


\paragraph{Training protocol.} All trials use AdamW with a batch size of 256, a maximum of 300 epochs, gradient clipping at 1.0, and mixed-precision (AMP) training. Early stopping is applied with a patience of 16 epochs on the primary validation metric.

\paragraph{Seed evaluation.} Once tuning completes, the Optuna-selected configuration for each (variant, dataset) pair is evaluated with 15 independent random seeds. The mean test metric across seeds is reported as the final score for that (variant, dataset) in Sec.~\ref{sec:exp:main}.






\section{Compute Resources}
\label{app:compute_resources}

\paragraph{Hardware and accounting scope.} All experiments were conducted on an internal on-premise GPU server equipped with a single NVIDIA H100 GPU (80GB VRAM). The reported results include both hyperparameter optimization and multi-seed evaluation across all datasets.

\paragraph{Compute time.} We provide a detailed breakdown of computational cost for each dataset in Tab.~\ref{tab:time_per_dataset}. The table reports hyperparameter tuning time, per-seed runtime (training + inference), and the number of random seeds used. This allows estimating both per-dataset and total compute requirements.

In total, hyperparameter tuning across all datasets requires approximately 64.8 GPU-hours, while multi-seed evaluation (15 seeds across 77 datasets) requires 14.1 GPU-hours, resulting in an overall compute cost of about 79.0 GPU-hours (3.3 GPU-days).
In our setup, hyperparameter optimization constitutes the dominant portion of the total compute, while individual training runs are relatively lightweight. Despite evaluating multiple seeds and using an ensemble of $K$ submodels, the overall computational cost remains manageable due to efficient parameterization.
All reported experiments correspond to the final configurations after hyperparameter search.

\input{tables/compute_time_per_dataset}

\section{Diversity Grid: Full Numerical Results}
\label{app:diversity_grid}
Tables~\ref{tab:diversity_grid_cls} and~\ref{tab:diversity_grid_reg} report the full numerical results of the diversity grid experiment summarized in Sec.~\ref{sec:exp:controllability_grid}, covering the four datasets featured in the main text. Each cell reports the mean and standard deviation over 15 seeds. Classification datasets report pairwise KL divergence, argmax disagreement, and test accuracy; regression datasets report Krogh--Vedelsby ambiguity and test RMSE. Full grid results for all remaining datasets are provided in the accompanying code repository.

\input{tables/tab_diversity_grid_cls}
\input{tables/tab_diversity_grid_reg}

\clearpage

\section{Selected Hyperparameters and Parameter Counts}
\label{app:n_parameters}

\paragraph{Expressivity without parameter overhead.} 

As established in Proposition~\ref{prop:expressivity}, LoMETab possesses a strictly larger hypothesis class than TabM; however, this increased expressivity does not necessarily entail a proportional increase in the number of parameters. While LoMETab introduces additional parameters through member-specific low-rank multiplicative factors, this increase is structurally constrained. Specifically, the adapter parameters scale as $Kr (d_{\text{in}} + d_{\text{out}})$, which is significantly smaller than the size of the full weight matrices and depends on the rank $r$ rather than the full model dimensionality. As a result, the additional capacity does not translate into a proportional growth in total parameter count.

Tab.~\ref{tab:hp_selection_dist} summarizes the distribution of hyperparameters across all 77 datasets, where $K$ denotes the number of submodels, $r$ the adapter rank, and $\sigma_{\text{init}}$ the initialization scale. The results show that larger ranks are frequently selected, while the ensemble size remains balanced between $K=16$ and $K=32$. This suggests that the additional expressivity provided by multiplicative low-rank adapters is consistently utilized across diverse datasets.

Tab.~\ref{tab:nparams_lometab_tabm} further provides a detailed comparison of the selected architectures and resulting parameter counts, where $d$ denotes the hidden dimension and $L$ the number of LoMETab linear layers. Although LoMETab can in principle utilize more parameters than TabM due to its extended hypothesis class, empirical results show that this is not always the case. In practice, hyperparameter optimization naturally adjusts other architectural factors—such as $d$ and $L$—often leading to more compact configurations.

Consequently, LoMETab does not consistently incur a higher parameter cost; instead, it frequently achieves comparable or even lower parameter counts than TabM across several datasets. This indicates that the additional expressivity introduced by the LoMETab formulation can be realized without a proportional increase in model size, highlighting a favorable trade-off between capacity and efficiency.

\input{tables/hp_selection_dist}

\input{tables/lometab_vs_tabm_table}

\section{Full Ablation Sweep Results}
\label{app:per-dataset-ablations}

Tables~\ref{tab:appendix_clf_acc}, \ref{tab:appendix_reg}, and \ref{tab:appendix_clf_ece} report the full per-dataset results of the hyperparameter sweep described in Sec.~\ref{sec:exp:design_axes}. For each dataset, we vary the rank $r \in \{1,2,4,8,16\}$ (with $K{=}16$, $\sigma_{\mathrm{init}}{=}1.0$ fixed) and the initialization scale $\sigma_{\mathrm{init}} \in \{0.1, 0.3, 0.5, 1.0, 2.0\}$ (with $K{=}16$, $r{=}16$ fixed), and compare against TabM with $K \in \{16, 32, 64, 128\}$ using its default hyperparameters. Tab.~\ref{tab:appendix_clf_acc} and Tab.~\ref{tab:appendix_reg} report accuracy and RMSE respectively, while Tab.~\ref{tab:appendix_clf_ece} reports the Expected Calibration Error (ECE) for classification datasets.
\input{tables/tab_appendix_clf_acc}
\input{tables/tab_appendix_reg.tex}
\input{tables/tab_appendix_clf_ece}

\clearpage

\section{Full Benchmark Results with Standard Deviations}
\label{app:benchmark_full_results}

The tables in this appendix provide the per-dataset benchmark scores used to construct
Fig.~\ref{fig:main_results}. We follow the reporting convention of the TabM benchmark: arrows
indicate whether higher or lower values are better, and the ``Single model'' and ``Ensemble''
columns correspond to the two evaluation settings reported in the released benchmark artifacts.
Baseline numbers are taken from the released TabM reports, while LoMETab results are computed in
our forked pipeline and averaged over 15 random seeds.

These tables are intended to make the source-dataset-level comparison transparent rather than to
claim statistical dominance. In several datasets, leading methods form a tight performance band, so
small gaps should be interpreted together with the reported standard deviations. A dash indicates
that the corresponding benchmark report does not provide an ensemble result; \texttt{nan} standard
deviations are inherited from the released reports when seed-level dispersion is unavailable.

\input{tables/per-dataset-why.tex}
\input{tables/per-dataset-default.tex}


\clearpage
\newpage

\end{document}

%% file: tables/tab_axis_impact.tex
\begin{table*}[t]
\centering\small
\setlength{\tabcolsep}{4pt}
\caption{
    \textbf{Performance sweep ablation: TabM $K$-scaling vs.\ LoMETab single-axis sweeps.}
    Sign-corrected $\Delta\%$ over (a) TabM $K{=}32$ baseline, evaluated on selected datasets (15 seeds).
    (d) Best LoMETab$^{r}$: best over $r \in \{1,2,4,8,16\}$, fixed $\sigma_{\mathrm{init}}{=}1.0$, $K{=}16$.
    (e) Best LoMETab$^{\sigma}$: best over $\sigma_{\mathrm{init}} \in \{0.1,0.3,0.5,1.0,2.0\}$, fixed $r{=}16$, $K{=}16$.
    Best $(r, \sigma_{\mathrm{init}})$ pair shown below each $\Delta\%$.
    \textbf{Bold}: positive $\Delta\%$.
}
\label{tab:axis_impact}
\resizebox{\linewidth}{!}{%
\begin{tabular}{lccccccccc}
\toprule
 & \multicolumn{5}{c}{Classification (Accuracy $\uparrow$)} & \multicolumn{4}{c}{Regression (RMSE $\downarrow$)} \\
\cmidrule(lr){2-6} \cmidrule(lr){7-10}
 & \textbf{HI} & \textbf{OT} & \textbf{CO} & \textbf{AD} & \textbf{CH} & \textbf{BL} & \textbf{CA} & \textbf{DI} & \textbf{MI} \\
\midrule
(a) Baseline
 & 0.729 & 0.822 & 0.933 & 0.861 & 0.846 & 0.687 & 0.443 & 0.133 & 0.745 \\
\midrule
\multicolumn{10}{c}{\textit{Vanilla TabM, $K$ varied}} \\
\midrule
(b) TabM$^{K{=}16}$
 & $-0.01\%$ & $-0.01\%$ & $\mathbf{+0.40\%}$ & $-0.07\%$ & $-0.03\%$ & $-0.20\%$ & $-4.65\%$ & $-2.33\%$ & $-0.76\%$ \\
(c) TabM$^{K{=}64}$
 & $-0.06\%$ & $-0.28\%$ & $-7.37\%$ & $-0.86\%$ & $-1.32\%$ & $-0.14\%$ & $-4.38\%$ & $-2.13\%$ & $-0.83\%$ \\
\midrule
\multicolumn{10}{c}{\textit{LoMETab ($K{=}16$, $\sigma{=}1.0$) fixed, $r$ varied}} \\
\midrule
(d) Best LoMETab$^{r}$
 & $\mathbf{+0.09\%}$ & $\mathbf{+0.41\%}$ & $\mathbf{+2.64\%}$ & $-0.21\%$ & $\mathbf{+1.75\%}$ & $\mathbf{+0.49\%}$ & $-0.77\%$ & $-0.40\%$ & $-0.22\%$ \\
 & \scriptsize$r{=}16$ & \scriptsize$r{=}16$ & \scriptsize$r{=}1$ & \scriptsize$r{=}2$ & \scriptsize$r{=}2$ & \scriptsize$r{=}16$ & \scriptsize$r{=}16$ & \scriptsize$r{=}2$ & \scriptsize$r{=}16$ \\
\midrule
\multicolumn{10}{c}{\textit{LoMETab ($K{=}16$, $r{=}16$) fixed, $\sigma$ varied}} \\
\midrule
(e) Best LoMETab$^{\sigma}$
 & $\mathbf{+0.09\%}$ & $\mathbf{+0.41\%}$ & $\mathbf{+2.56\%}$ & $-0.28\%$ & $\mathbf{+1.78\%}$ & $\mathbf{+0.52\%}$ & $\mathbf{+0.15\%}$ & $-0.30\%$ & $-0.22\%$ \\
 & \scriptsize$\sigma{=}1.0$ & \scriptsize$\sigma{=}1.0$ & \scriptsize$\sigma{=}0.3$ & \scriptsize$\sigma{=}0.3$ & \scriptsize$\sigma{=}0.5$ & \scriptsize$\sigma{=}0.5$ & \scriptsize$\sigma{=}2.0$ & \scriptsize$\sigma{=}0.5$ & \scriptsize$\sigma{=}1.0$ \\
\bottomrule
\addlinespace[2pt]
\multicolumn{10}{l}{\scriptsize\textit{Datasets:} HI=Higgs-Small, OT=Otto, CO=Covtype2, AD=Adult, CH=Churn; BL=Black-Friday, CA=California, DI=Diamond, MI=Microsoft.} \\
\end{tabular}}
\vspace{-1em}
\end{table*}

%% file: tables/tab_pareto.tex
\begin{table*}[t]
\centering\small
\setlength{\tabcolsep}{4pt}
\caption{
    \textbf{Accuracy--calibration ablation: TabM $K$-scaling vs.\ LoMETab best configuration.}
    Sign-corrected $\Delta\%$ on accuracy and ECE over (a) TabM $K{=}32$ baseline, evaluated on 5 classification datasets (15 seeds).
    LoMETab uses fixed $K{=}16$.
    (d) Best-ECE: single best $(r, \sigma_{\mathrm{init}})$ selected to minimize ECE per dataset.
    (e) Best-Acc: single best $(r, \sigma_{\mathrm{init}})$ selected to maximize accuracy per dataset.
    Search space: $r \in \{1,2,4,8,16\}$, $\sigma_{\mathrm{init}} \in \{0.1,0.3,0.5,1.0,2.0\}$.
    \textbf{Bold}: both $\Delta$Acc and $\Delta$ECE are positive.
    \fbox{$\cdot$}: Acc improves but ECE degrades.
}
\label{tab:pareto}
\resizebox{\linewidth}{!}{%
\begin{tabular}{lcccccccccc}
\toprule
 & \multicolumn{2}{c}{\textbf{HI}} & \multicolumn{2}{c}{\textbf{OT}} & \multicolumn{2}{c}{\textbf{CO}} & \multicolumn{2}{c}{\textbf{AD}} & \multicolumn{2}{c}{\textbf{CH}} \\
\cmidrule(lr){2-3} \cmidrule(lr){4-5} \cmidrule(lr){6-7} \cmidrule(lr){8-9} \cmidrule(lr){10-11}
 & $\Delta$Acc & $\Delta$ECE & $\Delta$Acc & $\Delta$ECE & $\Delta$Acc & $\Delta$ECE & $\Delta$Acc & $\Delta$ECE & $\Delta$Acc & $\Delta$ECE \\
\midrule
(a) Baseline
 & 0.729 & 0.0217 & 0.822 & 0.0434 & 0.933 & 0.0150 & 0.861 & 0.0240 & 0.846 & 0.0219 \\
\midrule
\multicolumn{11}{c}{\textit{Vanilla TabM, $K$ varied}} \\
\midrule
(b) TabM$^{K{=}64}$
 & $-0.1\%$ & $+41\%$ & $-0.3\%$ & $+30\%$ & $-7.4\%$ & $-8\%$ & $-0.9\%$ & $+55\%$ & $-1.3\%$ & $-282\%$ \\
(c) TabM$^{K{=}128}$
 & $-0.1\%$ & $+38\%$ & $-0.7\%$ & $+54\%$ & $-6.4\%$ & $-6\%$ & $-0.2\%$ & $+35\%$ & $-0.6\%$ & $-113\%$ \\
\midrule
\multicolumn{11}{c}{\textit{LoMETab ($K{=}16$), best over all $(r,\, \sigma_{\mathrm{init}})$ combinations}} \\
\midrule
(d) Best-ECE LoMETab
 & $-0.2\%$ & $+38\%$ & $\mathbf{+0.4\%}$ & $\mathbf{+78\%}$ & \fbox{$+2.6\%$} & $-15\%$ & $-0.3\%$ & $+60\%$ & $\mathbf{+1.5\%}$ & $\mathbf{+10\%}$ \\
 & \multicolumn{2}{c}{\scriptsize$(r{=}16,\,\sigma{=}2.0)$} & \multicolumn{2}{c}{\scriptsize$(r{=}16,\,\sigma{=}1.0)$} & \multicolumn{2}{c}{\scriptsize$(r{=}16,\,\sigma{=}0.3)$} & \multicolumn{2}{c}{\scriptsize$(r{=}16,\,\sigma{=}0.3)$} & \multicolumn{2}{c}{\scriptsize$(r{=}16,\,\sigma{=}0.3)$} \\
(e) Best-Acc LoMETab
 & $\mathbf{+0.1\%}$ & $\mathbf{+34\%}$ & $\mathbf{+0.4\%}$ & $\mathbf{+78\%}$ & \fbox{$+2.6\%$} & $-19\%$ & $-0.2\%$ & $+48\%$ & \fbox{$+1.8\%$} & $-12\%$ \\
 & \multicolumn{2}{c}{\scriptsize$(r{=}16,\,\sigma{=}1.0)$} & \multicolumn{2}{c}{\scriptsize$(r{=}16,\,\sigma{=}1.0)$} & \multicolumn{2}{c}{\scriptsize$(r{=}1,\,\sigma{=}1.0)$} & \multicolumn{2}{c}{\scriptsize$(r{=}2,\,\sigma{=}1.0)$} & \multicolumn{2}{c}{\scriptsize$(r{=}16,\,\sigma{=}0.5)$} \\
\bottomrule
\addlinespace[2pt]
\multicolumn{11}{l}{\scriptsize\textit{Datasets:} HI=Higgs-Small, OT=Otto, CO=Covtype2, AD=Adult, CH=Churn.} \\
\end{tabular}}
\end{table*}

%% file: tables/compute_time_per_dataset.tex
\begin{longtable}{llllr}
\caption{Per-dataset compute time for LoMETab. Tuning time is total HPO time (Optuna, up to 100 trials). Per-Seed time denotes training+inference time per seed over 15 seeds.} \label{tab:time_per_dataset} \\
\toprule
Dataset & Fold & Tuning Time (min) & Per-Seed Time (sec) & \# Seeds \\
\midrule
\endfirsthead
\caption[]{Per-dataset compute time for LoMETab. Tuning time is total HPO time (Optuna, up to 100 trials). Per-Seed time denotes training+inference time per seed over 15 seeds.} \\
\toprule
Dataset & Fold & Tuning Time (min) & Per-Seed Time (sec) & \# Seeds \\
\midrule
\endhead
\midrule
\multicolumn{5}{r}{Continued on next page} \\
\midrule
\endfoot
\bottomrule
\endlastfoot
adult & - & 64.5 & 16.7±1.5 & 15 \\
black-friday & - & 124.0 & 90.1±11.8 & 15 \\
california & - & 58.7 & 90.5±28.8 & 15 \\
churn & - & 30.7 & 11.4±2.3 & 15 \\
road-safety & 0 & 49.3 & 101.1±17.8 & 15 \\
KDDCup09\_upsell… & 0 & 13.9 & 4.4±0.7 & 15 \\
KDDCup09\_upsell… & 1 & 14.4 & 6.9±2.1 & 15 \\
KDDCup09\_upsell… & 2 & 13.2 & 6.3±1.7 & 15 \\
MiniBooNE & 0 & 112.3 & 18.3±2.9 & 15 \\
jannis & 0 & 42.7 & 19.0±3.6 & 15 \\
MagicTelescope & 0 & 29.8 & 15.9±2.0 & 15 \\
bank-marketing & 0 & 15.6 & 30.6±3.7 & 15 \\
credit & 0 & 25.8 & 16.9±1.9 & 15 \\
kdd\_ipums\_la\_97… & 0 & 20.9 & 7.2±1.5 & 15 \\
phoneme & 0 & 21.0 & 11.4±2.5 & 15 \\
wine & 0 & 15.1 & 12.5±1.6 & 15 \\
MagicTelescope & 1 & 24.1 & 13.3±2.0 & 15 \\
bank-marketing & 1 & 32.1 & 8.7±2.5 & 15 \\
credit & 1 & 21.3 & 9.2±2.2 & 15 \\
kdd\_ipums\_la\_97… & 1 & 13.7 & 7.4±1.3 & 15 \\
phoneme & 1 & 38.4 & 25.3±5.3 & 15 \\
wine & 1 & 29.7 & 20.2±7.4 & 15 \\
MagicTelescope & 2 & 22.0 & 32.8±4.4 & 15 \\
bank-marketing & 2 & 42.7 & 17.9±2.8 & 15 \\
kdd\_ipums\_la\_97… & 2 & 27.0 & 7.2±1.0 & 15 \\
phoneme & 2 & 9.2 & 11.9±2.1 & 15 \\
wine & 2 & 13.7 & 10.4±2.2 & 15 \\
phoneme & 3 & 16.8 & 7.9±2.5 & 15 \\
wine & 3 & 13.6 & 12.6±2.8 & 15 \\
phoneme & 4 & 15.5 & 9.6±1.2 & 15 \\
wine & 4 & 17.5 & 12.9±3.5 & 15 \\
covtype2 & - & 464.2 & 613.1±139.1 & 15 \\
diamond & - & 84.5 & 33.1±12.2 & 15 \\
higgs-small & - & 113.9 & 79.6±19.2 & 15 \\
microsoft & - & 259.7 & 172.9±40.8 & 15 \\
otto & - & 86.0 & 66.6±29.6 & 15 \\
nyc-taxi-green-… & 0 & 199.0 & 134.9±25.1 & 15 \\
particulate-mat… & 0 & 108.7 & 184.8±45.7 & 15 \\
Brazilian\_house… & 0 & 33.3 & 17.8±9.9 & 15 \\
Mercedes\_Benz\_G… & 0 & 19.7 & 7.7±1.4 & 15 \\
OnlineNewsPopul… & 0 & 25.9 & 23.5±3.0 & 15 \\
analcatdata\_sup… & 0 & 19.5 & 7.0±0.9 & 15 \\
house\_sales & 0 & 41.5 & 37.3±10.2 & 15 \\
Brazilian\_house… & 1 & 30.1 & 69.2±5.1 & 15 \\
Mercedes\_Benz\_G… & 1 & 20.6 & 6.2±1.5 & 15 \\
analcatdata\_sup… & 1 & 18.5 & 8.8±2.1 & 15 \\
Brazilian\_house… & 2 & 42.3 & 26.8±8.3 & 15 \\
Mercedes\_Benz\_G… & 2 & 18.2 & 7.3±1.4 & 15 \\
analcatdata\_sup… & 2 & 22.0 & 19.0±7.5 & 15 \\
Mercedes\_Benz\_G… & 3 & 21.1 & 4.6±1.0 & 15 \\
analcatdata\_sup… & 3 & 40.1 & 24.4±14.1 & 15 \\
Mercedes\_Benz\_G… & 4 & 26.5 & 9.3±2.7 & 15 \\
analcatdata\_sup… & 4 & 20.1 & 27.6±13.3 & 15 \\
year & 0 & 101.6 & 57.2±11.5 & 15 \\
Ailerons & 0 & 52.2 & 33.4±12.2 & 15 \\
MiamiHousing201… & 0 & 41.5 & 38.1±11.9 & 15 \\
cpu\_act & 0 & 45.6 & 50.0±8.7 & 15 \\
elevators & 0 & 40.4 & 35.4±7.8 & 15 \\
fifa & 0 & 61.1 & 56.1±9.9 & 15 \\
isolet & 0 & 60.5 & 91.1±29.6 & 15 \\
medical\_charges & 0 & 69.4 & 50.2±14.2 & 15 \\
pol & 0 & 62.7 & 44.7±21.4 & 15 \\
superconduct & 0 & 45.5 & 62.1±9.6 & 15 \\
wine\_quality & 0 & 30.8 & 26.4±11.4 & 15 \\
Ailerons & 1 & 66.2 & 37.4±9.4 & 15 \\
MiamiHousing201… & 1 & 32.8 & 30.0±8.2 & 15 \\
cpu\_act & 1 & 68.5 & 95.1±18.6 & 15 \\
elevators & 1 & 41.2 & 42.7±8.9 & 15 \\
fifa & 1 & 64.5 & 55.8±14.5 & 15 \\
isolet & 1 & 33.9 & 56.5±40.7 & 15 \\
pol & 1 & 45.8 & 34.2±11.2 & 15 \\
wine\_quality & 1 & 37.1 & 36.2±9.0 & 15 \\
Ailerons & 2 & 32.7 & 26.0±6.4 & 15 \\
MiamiHousing201… & 2 & 50.1 & 21.5±6.5 & 15 \\
cpu\_act & 2 & 31.6 & 29.6±6.5 & 15 \\
isolet & 2 & 21.0 & 29.7±25.7 & 15 \\
wine\_quality & 2 & 49.2 & 103.5±33.1 & 15 \\
\end{longtable}

%% file: tables/tab_diversity_grid_cls.tex
\begin{table}[h!]
\centering\small
\caption{Diversity grid: classification results (mean $\pm$ std over 15 seeds).}
\label{tab:diversity_grid_cls}
\resizebox{\linewidth}{!}{%
\begin{tabular}{cc|ccc|ccc}
\toprule
& & \multicolumn{3}{c|}{\textbf{Adult}} & \multicolumn{3}{c}{\textbf{Higgs-Small}} \\
$r$ & $\sigma$ & KL $\uparrow$ & Disagg. $\uparrow$ & Acc. $\uparrow$ & KL $\uparrow$ & Disagg. $\uparrow$ & Acc. $\uparrow$ \\
\midrule
1 & 0.1 & $0.001 \pm 0.000$ & $0.008 \pm 0.001$ & $0.865 \pm 0.001$ & $0.000 \pm 0.000$ & $0.008 \pm 0.001$ & $0.720 \pm 0.001$ \\
1 & 0.3 & $0.002 \pm 0.001$ & $0.013 \pm 0.002$ & $0.865 \pm 0.002$ & $0.000 \pm 0.000$ & $0.012 \pm 0.001$ & $0.721 \pm 0.002$ \\
1 & 0.5 & $0.003 \pm 0.001$ & $0.015 \pm 0.002$ & $0.865 \pm 0.001$ & $0.001 \pm 0.000$ & $0.022 \pm 0.001$ & $0.720 \pm 0.003$ \\
1 & 1.0 & $0.006 \pm 0.001$ & $0.024 \pm 0.002$ & $0.865 \pm 0.002$ & $0.006 \pm 0.001$ & $0.053 \pm 0.003$ & $0.722 \pm 0.002$ \\
1 & 2.0 & $0.020 \pm 0.004$ & $0.045 \pm 0.005$ & $0.868 \pm 0.001$ & $0.035 \pm 0.004$ & $0.130 \pm 0.006$ & $0.728 \pm 0.001$ \\
1 & 3.0 & $0.028 \pm 0.006$ & $0.051 \pm 0.007$ & $0.868 \pm 0.002$ & $0.042 \pm 0.008$ & $0.141 \pm 0.011$ & $0.728 \pm 0.002$ \\
\midrule
2 & 0.1 & $0.001 \pm 0.000$ & $0.008 \pm 0.001$ & $0.865 \pm 0.002$ & $0.000 \pm 0.000$ & $0.008 \pm 0.001$ & $0.720 \pm 0.002$ \\
2 & 0.3 & $0.002 \pm 0.000$ & $0.012 \pm 0.001$ & $0.864 \pm 0.002$ & $0.000 \pm 0.000$ & $0.015 \pm 0.001$ & $0.721 \pm 0.002$ \\
2 & 0.5 & $0.003 \pm 0.001$ & $0.017 \pm 0.002$ & $0.865 \pm 0.002$ & $0.002 \pm 0.000$ & $0.029 \pm 0.002$ & $0.721 \pm 0.002$ \\
2 & 1.0 & $0.013 \pm 0.002$ & $0.036 \pm 0.003$ & $0.867 \pm 0.002$ & $0.015 \pm 0.002$ & $0.085 \pm 0.004$ & $0.726 \pm 0.002$ \\
2 & 2.0 & $0.030 \pm 0.009$ & $0.053 \pm 0.007$ & $0.868 \pm 0.002$ & $0.048 \pm 0.007$ & $0.146 \pm 0.008$ & $0.729 \pm 0.002$ \\
2 & 3.0 & $0.037 \pm 0.007$ & $0.057 \pm 0.005$ & $0.869 \pm 0.002$ & $0.046 \pm 0.010$ & $0.144 \pm 0.011$ & $0.729 \pm 0.002$ \\
\midrule
4 & 0.1 & $0.001 \pm 0.000$ & $0.008 \pm 0.001$ & $0.864 \pm 0.003$ & $0.000 \pm 0.000$ & $0.009 \pm 0.001$ & $0.720 \pm 0.002$ \\
4 & 0.3 & $0.002 \pm 0.000$ & $0.013 \pm 0.002$ & $0.865 \pm 0.003$ & $0.001 \pm 0.000$ & $0.019 \pm 0.002$ & $0.720 \pm 0.002$ \\
4 & 0.5 & $0.004 \pm 0.001$ & $0.019 \pm 0.002$ & $0.865 \pm 0.002$ & $0.003 \pm 0.000$ & $0.036 \pm 0.002$ & $0.722 \pm 0.001$ \\
4 & 1.0 & $0.021 \pm 0.006$ & $0.045 \pm 0.005$ & $0.868 \pm 0.002$ & $0.036 \pm 0.007$ & $0.125 \pm 0.009$ & $0.729 \pm 0.001$ \\
4 & 2.0 & $0.032 \pm 0.006$ & $0.053 \pm 0.005$ & $0.868 \pm 0.002$ & $0.053 \pm 0.009$ & $0.151 \pm 0.008$ & $0.730 \pm 0.002$ \\
4 & 3.0 & $0.037 \pm 0.009$ & $0.058 \pm 0.006$ & $0.869 \pm 0.002$ & $0.053 \pm 0.008$ & $0.150 \pm 0.009$ & $0.728 \pm 0.001$ \\
\midrule
8 & 0.1 & $0.001 \pm 0.000$ & $0.008 \pm 0.002$ & $0.864 \pm 0.002$ & $0.000 \pm 0.000$ & $0.009 \pm 0.001$ & $0.720 \pm 0.003$ \\
8 & 0.3 & $0.002 \pm 0.000$ & $0.014 \pm 0.002$ & $0.865 \pm 0.002$ & $0.001 \pm 0.000$ & $0.023 \pm 0.002$ & $0.719 \pm 0.003$ \\
8 & 0.5 & $0.006 \pm 0.001$ & $0.024 \pm 0.002$ & $0.865 \pm 0.003$ & $0.004 \pm 0.001$ & $0.045 \pm 0.003$ & $0.722 \pm 0.002$ \\
8 & 1.0 & $0.021 \pm 0.006$ & $0.043 \pm 0.007$ & $0.867 \pm 0.002$ & $0.048 \pm 0.010$ & $0.144 \pm 0.011$ & $0.730 \pm 0.001$ \\
8 & 2.0 & $0.032 \pm 0.009$ & $0.053 \pm 0.009$ & $0.868 \pm 0.003$ & $0.063 \pm 0.016$ & $0.161 \pm 0.013$ & $0.730 \pm 0.001$ \\
8 & 3.0 & $0.038 \pm 0.007$ & $0.056 \pm 0.006$ & $0.869 \pm 0.002$ & $0.060 \pm 0.016$ & $0.154 \pm 0.012$ & $0.727 \pm 0.002$ \\
\midrule
16 & 0.1 & $0.001 \pm 0.000$ & $0.009 \pm 0.002$ & $0.864 \pm 0.002$ & $0.000 \pm 0.000$ & $0.010 \pm 0.001$ & $0.720 \pm 0.003$ \\
16 & 0.3 & $0.003 \pm 0.001$ & $0.017 \pm 0.002$ & $0.865 \pm 0.001$ & $0.002 \pm 0.000$ & $0.029 \pm 0.002$ & $0.721 \pm 0.002$ \\
16 & 0.5 & $0.010 \pm 0.003$ & $0.030 \pm 0.003$ & $0.867 \pm 0.002$ & $0.009 \pm 0.001$ & $0.067 \pm 0.007$ & $0.723 \pm 0.003$ \\
16 & 1.0 & $0.022 \pm 0.007$ & $0.044 \pm 0.006$ & $0.867 \pm 0.002$ & $0.063 \pm 0.016$ & $0.154 \pm 0.015$ & $0.730 \pm 0.002$ \\
16 & 2.0 & $0.030 \pm 0.006$ & $0.051 \pm 0.005$ & $0.868 \pm 0.002$ & $0.072 \pm 0.016$ & $0.165 \pm 0.014$ & $0.730 \pm 0.001$ \\
16 & 3.0 & $0.042 \pm 0.011$ & $0.059 \pm 0.006$ & $0.869 \pm 0.002$ & $0.069 \pm 0.015$ & $0.161 \pm 0.009$ & $0.726 \pm 0.001$ \\
\bottomrule
\end{tabular}
}
\end{table}

%% file: tables/tab_diversity_grid_reg.tex
\begin{table}[h!]
\centering\small
\caption{Diversity grid: regression results (mean $\pm$ std over 15 seeds).}
\label{tab:diversity_grid_reg}
\begin{tabular}{cc|cc|cc}
\toprule
& & \multicolumn{2}{c|}{\textbf{California}} & \multicolumn{2}{c}{\textbf{Black-Friday}} \\
$r$ & $\sigma$ & Ambiguity $\uparrow$ & RMSE $\downarrow$ & Ambiguity $\uparrow$ & RMSE $\downarrow$ \\
\midrule
1 & 0.1 & $0.001 \pm 0.000$ & $0.472 \pm 0.007$ & $0.001 \pm 0.000$ & $0.692 \pm 0.001$ \\
1 & 0.3 & $0.001 \pm 0.000$ & $0.475 \pm 0.005$ & $0.001 \pm 0.000$ & $0.692 \pm 0.001$ \\
1 & 0.5 & $0.002 \pm 0.000$ & $0.471 \pm 0.004$ & $0.001 \pm 0.000$ & $0.692 \pm 0.001$ \\
1 & 1.0 & $0.012 \pm 0.002$ & $0.455 \pm 0.002$ & $0.005 \pm 0.001$ & $0.688 \pm 0.001$ \\
1 & 2.0 & $0.043 \pm 0.009$ & $0.446 \pm 0.003$ & $0.020 \pm 0.002$ & $0.684 \pm 0.001$ \\
1 & 3.0 & $0.044 \pm 0.007$ & $0.443 \pm 0.004$ & $0.024 \pm 0.003$ & $0.683 \pm 0.000$ \\
\midrule
2 & 0.1 & $0.001 \pm 0.000$ & $0.472 \pm 0.005$ & $0.001 \pm 0.000$ & $0.691 \pm 0.001$ \\
2 & 0.3 & $0.001 \pm 0.000$ & $0.472 \pm 0.004$ & $0.001 \pm 0.000$ & $0.692 \pm 0.001$ \\
2 & 0.5 & $0.002 \pm 0.000$ & $0.468 \pm 0.004$ & $0.001 \pm 0.000$ & $0.691 \pm 0.001$ \\
2 & 1.0 & $0.029 \pm 0.006$ & $0.450 \pm 0.003$ & $0.014 \pm 0.002$ & $0.685 \pm 0.001$ \\
2 & 2.0 & $0.046 \pm 0.006$ & $0.446 \pm 0.002$ & $0.020 \pm 0.002$ & $0.684 \pm 0.000$ \\
2 & 3.0 & $0.045 \pm 0.009$ & $0.440 \pm 0.002$ & $0.022 \pm 0.004$ & $0.683 \pm 0.001$ \\
\midrule
4 & 0.1 & $0.001 \pm 0.000$ & $0.473 \pm 0.006$ & $0.001 \pm 0.000$ & $0.691 \pm 0.001$ \\
4 & 0.3 & $0.001 \pm 0.000$ & $0.473 \pm 0.005$ & $0.001 \pm 0.000$ & $0.691 \pm 0.002$ \\
4 & 0.5 & $0.004 \pm 0.001$ & $0.467 \pm 0.004$ & $0.002 \pm 0.000$ & $0.690 \pm 0.001$ \\
4 & 1.0 & $0.033 \pm 0.009$ & $0.448 \pm 0.003$ & $0.017 \pm 0.002$ & $0.685 \pm 0.001$ \\
4 & 2.0 & $0.046 \pm 0.005$ & $0.445 \pm 0.003$ & $0.019 \pm 0.002$ & $0.683 \pm 0.001$ \\
4 & 3.0 & $0.041 \pm 0.006$ & $0.438 \pm 0.002$ & $0.015 \pm 0.002$ & $0.682 \pm 0.001$ \\
\midrule
8 & 0.1 & $0.001 \pm 0.000$ & $0.474 \pm 0.005$ & $0.001 \pm 0.000$ & $0.690 \pm 0.001$ \\
8 & 0.3 & $0.002 \pm 0.000$ & $0.470 \pm 0.003$ & $0.001 \pm 0.000$ & $0.691 \pm 0.002$ \\
8 & 0.5 & $0.008 \pm 0.003$ & $0.458 \pm 0.003$ & $0.004 \pm 0.001$ & $0.688 \pm 0.001$ \\
8 & 1.0 & $0.036 \pm 0.006$ & $0.450 \pm 0.003$ & $0.017 \pm 0.003$ & $0.685 \pm 0.001$ \\
8 & 2.0 & $0.045 \pm 0.008$ & $0.443 \pm 0.003$ & $0.020 \pm 0.003$ & $0.683 \pm 0.001$ \\
8 & 3.0 & $0.037 \pm 0.005$ & $0.439 \pm 0.003$ & $0.012 \pm 0.002$ & $0.683 \pm 0.001$ \\
\midrule
16 & 0.1 & $0.001 \pm 0.000$ & $0.474 \pm 0.004$ & $0.001 \pm 0.000$ & $0.690 \pm 0.001$ \\
16 & 0.3 & $0.003 \pm 0.001$ & $0.472 \pm 0.002$ & $0.002 \pm 0.000$ & $0.690 \pm 0.001$ \\
16 & 0.5 & $0.022 \pm 0.008$ & $0.452 \pm 0.003$ & $0.010 \pm 0.002$ & $0.685 \pm 0.001$ \\
16 & 1.0 & $0.034 \pm 0.008$ & $0.450 \pm 0.004$ & $0.016 \pm 0.003$ & $0.685 \pm 0.001$ \\
16 & 2.0 & $0.041 \pm 0.007$ & $0.441 \pm 0.003$ & $0.016 \pm 0.002$ & $0.683 \pm 0.001$ \\
16 & 3.0 & $0.033 \pm 0.004$ & $0.441 \pm 0.007$ & $0.036 \pm 0.052$ & $0.693 \pm 0.013$ \\
\bottomrule
\end{tabular}
\end{table}

%% file: tables/hp_selection_dist.tex
\begin{table}[t]
\centering\small
\setlength{\tabcolsep}{5pt}
\caption{Distribution of $r$ and $\sigma_{\mathrm{init}}$ across all 77 datasets, broken down by $K$.}
\begin{subfigure}[t]{0.48\linewidth}
    \centering
    \caption{Selected adapter rank $r$}
    \begin{tabular}{l|ccccc|c}
        \toprule
        $K$ & $1$ & $2$ & $4$ & $8$ & $16$ & \textbf{All} \\
        \midrule
        $K=16$ & 5 & 3 & 6 & 10 & 12 & 36 \\
        $K=32$ & 4 & 4 & 9 & 9 & 15 & 41 \\
        \midrule
        All & 9 & 7 & 15 & 19 & 27 & 77 \\
        \bottomrule
    \end{tabular}
\end{subfigure}\hfill
\begin{subfigure}[t]{0.48\linewidth}
    \centering
    \caption{Selected initialization scale $\sigma_{\mathrm{init}}$}
    \begin{tabular}{l|cccc|c}
        \toprule
        $K$ & $0.1$ & $0.3$ & $0.5$ & $1.0$ & \textbf{All} \\
        \midrule
        $K=16$ & 2 & 5 & 6 & 23 & 36 \\
        $K=32$ & 1 & 10 & 7 & 23 & 41 \\
        \midrule
        All & 3 & 15 & 13 & 46 & 77 \\
        \bottomrule
    \end{tabular}
\end{subfigure}
\label{tab:hp_selection_dist}
\end{table}

%% file: tables/lometab_vs_tabm_table.tex
\begin{longtable}{l l l ccccc l}
\caption{Parameter Comparison between LoMETab and TabM} \label{tab:nparams_lometab_tabm} \\
\toprule
Dataset & Fold & Model & $K$ & $r$ & $d$ & $L$ & \#params & Saving \\
\midrule
\endfirsthead
\caption*{Parameter Comparison between LoMETab and TabM (continued)} \\
\toprule
Dataset & Fold & Model & $K$ & $r$ & $d$ & $L$ & \#params & Saving \\
\midrule
\endhead
\midrule
\multicolumn{9}{r}{\textit{Continued on next page}} \\
\endfoot
\bottomrule
\endlastfoot
\midrule
\multirow{2}{*}{AD} & \multirow{2}{*}{-} & LoMETab & 32 & 16 & 272 & 1 & $409.5$K & \multirow{2}{*}{$+14.4$\%} \\
& & TabM & 32 & -- & 320 & 4 & $478.4$K & \\
\midrule
\multirow{2}{*}{Ailerons} & \multirow{2}{*}{0} & LoMETab & 32 & 16 & 64 & 3 & $510.3$K & \multirow{2}{*}{$+66.9$\%} \\
& & TabM & 32 & -- & 560 & 5 & $1.54$M & \\
\midrule
\multirow{2}{*}{Ailerons} & \multirow{2}{*}{1} & LoMETab & 32 & 4 & 128 & 2 & $157.1$K & \multirow{2}{*}{$-516.2$\%} \\
& & TabM & 32 & -- & 80 & 2 & $25.5$K & \\
\midrule
\multirow{2}{*}{Ailerons} & \multirow{2}{*}{2} & LoMETab & 32 & 16 & 64 & 3 & $533.5$K & \multirow{2}{*}{$+52.2$\%} \\
& & TabM & 32 & -- & 544 & 4 & $1.12$M & \\
\midrule
\multirow{2}{*}{BL} & \multirow{2}{*}{-} & LoMETab & 32 & 16 & 912 & 2 & $2.39$M & \multirow{2}{*}{$-88.4$\%} \\
& & TabM & 32 & -- & 1024 & 2 & $1.27$M & \\
\midrule
\multirow{2}{*}{Brazilian\_hous…} & \multirow{2}{*}{0} & LoMETab & 32 & 1 & 704 & 1 & $242.8$K & \multirow{2}{*}{$-115.5$\%} \\
& & TabM & 32 & -- & 992 & 1 & $112.7$K & \\
\midrule
\multirow{2}{*}{Brazilian\_hous…} & \multirow{2}{*}{1} & LoMETab & 32 & 4 & 992 & 1 & $310.5$K & \multirow{2}{*}{$-82.8$\%} \\
& & TabM & 32 & -- & 320 & 2 & $169.9$K & \\
\midrule
\multirow{2}{*}{Brazilian\_hous…} & \multirow{2}{*}{2} & LoMETab & 16 & 2 & 480 & 3 & $692.0$K & \multirow{2}{*}{$-2,573.0$\%} \\
& & TabM & 32 & -- & 224 & 1 & $25.9$K & \\
\midrule
\multirow{2}{*}{CA} & \multirow{2}{*}{-} & LoMETab & 32 & 16 & 416 & 2 & $1.02$M & \multirow{2}{*}{$-132.5$\%} \\
& & TabM & 32 & -- & 400 & 3 & $438.7$K & \\
\midrule
\multirow{2}{*}{CH} & \multirow{2}{*}{-} & LoMETab & 32 & 4 & 512 & 1 & $217.3$K & \multirow{2}{*}{$-11.9$\%} \\
& & TabM & 32 & -- & 192 & 4 & $194.1$K & \\
\midrule
\multirow{2}{*}{CO} & \multirow{2}{*}{-} & LoMETab & 16 & 16 & 944 & 3 & $3.35$M & \multirow{2}{*}{$+23.7$\%} \\
& & TabM & 32 & -- & 960 & 5 & $4.39$M & \\
\midrule
\multirow{2}{*}{DI} & \multirow{2}{*}{-} & LoMETab & 16 & 8 & 160 & 4 & $305.4$K & \multirow{2}{*}{$+40.1$\%} \\
& & TabM & 32 & -- & 432 & 3 & $509.8$K & \\
\midrule
\multirow{2}{*}{HI} & \multirow{2}{*}{-} & LoMETab & 32 & 16 & 80 & 3 & $526.2$K & \multirow{2}{*}{$+83.0$\%} \\
& & TabM & 32 & -- & 944 & 4 & $3.09$M & \\
\midrule
\multirow{2}{*}{KDDCup09\_upsel…} & \multirow{2}{*}{0} & LoMETab & 16 & 8 & 464 & 3 & $1.06$M & \multirow{2}{*}{$-19.8$\%} \\
& & TabM & 32 & -- & 464 & 4 & $889.1$K & \\
\midrule
\multirow{2}{*}{KDDCup09\_upsel…} & \multirow{2}{*}{1} & LoMETab & 16 & 4 & 128 & 2 & $224.1$K & \multirow{2}{*}{$+86.7$\%} \\
& & TabM & 32 & -- & 576 & 5 & $1.69$M & \\
\midrule
\multirow{2}{*}{KDDCup09\_upsel…} & \multirow{2}{*}{2} & LoMETab & 16 & 1 & 528 & 2 & $811.5$K & \multirow{2}{*}{$-583.1$\%} \\
& & TabM & 32 & -- & 128 & 4 & $118.8$K & \\
\midrule
\multirow{2}{*}{MI} & \multirow{2}{*}{-} & LoMETab & 16 & 16 & 224 & 4 & $1.33$M & \multirow{2}{*}{$+12.5$\%} \\
& & TabM & 32 & -- & 544 & 5 & $1.53$M & \\
\midrule
\multirow{2}{*}{MagicTelescope} & \multirow{2}{*}{0} & LoMETab & 32 & 8 & 272 & 4 & $821.7$K & \multirow{2}{*}{$+58.4$\%} \\
& & TabM & 32 & -- & 640 & 5 & $1.97$M & \\
\midrule
\multirow{2}{*}{MagicTelescope} & \multirow{2}{*}{1} & LoMETab & 16 & 16 & 480 & 1 & $289.9$K & \multirow{2}{*}{$+69.5$\%} \\
& & TabM & 32 & -- & 864 & 2 & $949.1$K & \\
\midrule
\multirow{2}{*}{MagicTelescope} & \multirow{2}{*}{2} & LoMETab & 16 & 8 & 544 & 4 & $1.48$M & \multirow{2}{*}{$+4.6$\%} \\
& & TabM & 32 & -- & 560 & 5 & $1.55$M & \\
\midrule
\multirow{2}{*}{Mercedes\_Benz} & \multirow{2}{*}{0} & LoMETab & 32 & 4 & 368 & 4 & $982.9$K & \multirow{2}{*}{$+38.1$\%} \\
& & TabM & 32 & -- & 560 & 4 & $1.59$M & \\
\midrule
\multirow{2}{*}{Mercedes\_Benz} & \multirow{2}{*}{1} & LoMETab & 32 & 8 & 480 & 4 & $1.91$M & \multirow{2}{*}{$-20.2$\%} \\
& & TabM & 32 & -- & 560 & 4 & $1.59$M & \\
\midrule
\multirow{2}{*}{Mercedes\_Benz} & \multirow{2}{*}{2} & LoMETab & 16 & 8 & 64 & 4 & $147.5$K & \multirow{2}{*}{$+95.7$\%} \\
& & TabM & 32 & -- & 784 & 5 & $3.43$M & \\
\midrule
\multirow{2}{*}{Mercedes\_Benz} & \multirow{2}{*}{3} & LoMETab & 32 & 8 & 816 & 2 & $1.78$M & \multirow{2}{*}{$+5.5$\%} \\
& & TabM & 32 & -- & 624 & 4 & $1.88$M & \\
\midrule
\multirow{2}{*}{Mercedes\_Benz} & \multirow{2}{*}{4} & LoMETab & 32 & 16 & 608 & 3 & $2.80$M & \multirow{2}{*}{$-42.5$\%} \\
& & TabM & 32 & -- & 640 & 4 & $1.96$M & \\
\midrule
\multirow{2}{*}{MiamiHousing201…} & \multirow{2}{*}{0} & LoMETab & 16 & 16 & 64 & 4 & $290.5$K & \multirow{2}{*}{$+88.5$\%} \\
& & TabM & 32 & -- & 736 & 5 & $2.53$M & \\
\midrule
\multirow{2}{*}{MiamiHousing201…} & \multirow{2}{*}{1} & LoMETab & 16 & 2 & 192 & 4 & $226.8$K & \multirow{2}{*}{$+31.4$\%} \\
& & TabM & 32 & -- & 272 & 4 & $330.4$K & \\
\midrule
\multirow{2}{*}{MiamiHousing201…} & \multirow{2}{*}{2} & LoMETab & 32 & 2 & 272 & 3 & $385.6$K & \multirow{2}{*}{$-192.2$\%} \\
& & TabM & 32 & -- & 192 & 3 & $132.0$K & \\
\midrule
\multirow{2}{*}{MiniBooNE} & \multirow{2}{*}{0} & LoMETab & 16 & 8 & 160 & 3 & $533.5$K & \multirow{2}{*}{$+84.8$\%} \\
& & TabM & 32 & -- & 1008 & 4 & $3.52$M & \\
\midrule
\multirow{2}{*}{OT} & \multirow{2}{*}{-} & LoMETab & 32 & 16 & 560 & 2 & $4.19$M & \multirow{2}{*}{$-65.8$\%} \\
& & TabM & 32 & -- & 976 & 3 & $2.53$M & \\
\midrule
\multirow{2}{*}{OnlineNewsPopul…} & \multirow{2}{*}{0} & LoMETab & 16 & 16 & 64 & 2 & $309.2$K & \multirow{2}{*}{$-76.3$\%} \\
& & TabM & 32 & -- & 304 & 2 & $175.3$K & \\
\midrule
\multirow{2}{*}{analcatdata\_su…} & \multirow{2}{*}{0} & LoMETab & 32 & 1 & 928 & 4 & $2.97$M & \multirow{2}{*}{$-1,825.1$\%} \\
& & TabM & 32 & -- & 144 & 5 & $154.2$K & \\
\midrule
\multirow{2}{*}{analcatdata\_su…} & \multirow{2}{*}{1} & LoMETab & 16 & 2 & 144 & 4 & $111.8$K & \multirow{2}{*}{$-71.0$\%} \\
& & TabM & 32 & -- & 80 & 5 & $65.4$K & \\
\midrule
\multirow{2}{*}{analcatdata\_su…} & \multirow{2}{*}{2} & LoMETab & 32 & 8 & 912 & 1 & $375.0$K & \multirow{2}{*}{$-467.5$\%} \\
& & TabM & 32 & -- & 96 & 4 & $66.1$K & \\
\midrule
\multirow{2}{*}{analcatdata\_su…} & \multirow{2}{*}{3} & LoMETab & 32 & 16 & 384 & 3 & $1.38$M & \multirow{2}{*}{$+65.8$\%} \\
& & TabM & 32 & -- & 944 & 5 & $4.03$M & \\
\midrule
\multirow{2}{*}{analcatdata\_su…} & \multirow{2}{*}{4} & LoMETab & 16 & 8 & 1008 & 1 & $204.5$K & \multirow{2}{*}{$+93.8$\%} \\
& & TabM & 32 & -- & 848 & 5 & $3.29$M & \\
\midrule
\multirow{2}{*}{bank-marketing} & \multirow{2}{*}{0} & LoMETab & 16 & 1 & 848 & 3 & $1.78$M & \multirow{2}{*}{$-125.3$\%} \\
& & TabM & 32 & -- & 384 & 5 & $789.4$K & \\
\midrule
\multirow{2}{*}{bank-marketing} & \multirow{2}{*}{1} & LoMETab & 32 & 8 & 576 & 1 & $317.9$K & \multirow{2}{*}{$+80.5$\%} \\
& & TabM & 32 & -- & 576 & 5 & $1.63$M & \\
\midrule
\multirow{2}{*}{bank-marketing} & \multirow{2}{*}{2} & LoMETab & 32 & 1 & 688 & 3 & $1.26$M & \multirow{2}{*}{$+9.4$\%} \\
& & TabM & 32 & -- & 528 & 5 & $1.39$M & \\
\midrule
\multirow{2}{*}{cpu\_act} & \multirow{2}{*}{0} & LoMETab & 32 & 16 & 112 & 4 & $897.4$K & \multirow{2}{*}{$+4.5$\%} \\
& & TabM & 32 & -- & 496 & 4 & $939.6$K & \\
\midrule
\multirow{2}{*}{cpu\_act} & \multirow{2}{*}{1} & LoMETab & 32 & 16 & 80 & 3 & $461.7$K & \multirow{2}{*}{$+81.0$\%} \\
& & TabM & 32 & -- & 720 & 5 & $2.44$M & \\
\midrule
\multirow{2}{*}{cpu\_act} & \multirow{2}{*}{2} & LoMETab & 16 & 16 & 176 & 3 & $430.6$K & \multirow{2}{*}{$+67.0$\%} \\
& & TabM & 32 & -- & 512 & 5 & $1.31$M & \\
\midrule
\multirow{2}{*}{credit} & \multirow{2}{*}{0} & LoMETab & 16 & 4 & 448 & 3 & $699.9$K & \multirow{2}{*}{$-3,497.5$\%} \\
& & TabM & 32 & -- & 64 & 2 & $19.5$K & \\
\midrule
\multirow{2}{*}{credit} & \multirow{2}{*}{1} & LoMETab & 16 & 4 & 544 & 4 & $1.38$M & \multirow{2}{*}{$+33.1$\%} \\
& & TabM & 32 & -- & 656 & 5 & $2.06$M & \\
\midrule
\multirow{2}{*}{elevators} & \multirow{2}{*}{0} & LoMETab & 32 & 4 & 208 & 3 & $441.4$K & \multirow{2}{*}{$-315.3$\%} \\
& & TabM & 32 & -- & 112 & 5 & $106.3$K & \\
\midrule
\multirow{2}{*}{elevators} & \multirow{2}{*}{1} & LoMETab & 16 & 1 & 768 & 4 & $2.13$M & \multirow{2}{*}{$-217.8$\%} \\
& & TabM & 32 & -- & 352 & 5 & $670.8$K & \\
\midrule
\multirow{2}{*}{fifa} & \multirow{2}{*}{0} & LoMETab & 32 & 2 & 272 & 2 & $186.3$K & \multirow{2}{*}{$-148.8$\%} \\
& & TabM & 32 & -- & 192 & 2 & $74.9$K & \\
\midrule
\multirow{2}{*}{fifa} & \multirow{2}{*}{1} & LoMETab & 32 & 8 & 112 & 2 & $173.0$K & \multirow{2}{*}{$+22.5$\%} \\
& & TabM & 32 & -- & 384 & 2 & $223.3$K & \\
\midrule
\multirow{2}{*}{house\_sales} & \multirow{2}{*}{0} & LoMETab & 32 & 4 & 64 & 2 & $188.4$K & \multirow{2}{*}{$-86.6$\%} \\
& & TabM & 32 & -- & 160 & 3 & $101.0$K & \\
\midrule
\multirow{2}{*}{isolet} & \multirow{2}{*}{0} & LoMETab & 32 & 1 & 544 & 4 & $8.69$M & \multirow{2}{*}{$-119.9$\%} \\
& & TabM & 32 & -- & 864 & 5 & $3.95$M & \\
\midrule
\multirow{2}{*}{isolet} & \multirow{2}{*}{1} & LoMETab & 16 & 1 & 752 & 3 & $17.48$M & \multirow{2}{*}{$-278.7$\%} \\
& & TabM & 32 & -- & 944 & 5 & $4.62$M & \\
\midrule
\multirow{2}{*}{isolet} & \multirow{2}{*}{2} & LoMETab & 16 & 16 & 512 & 3 & $11.28$M & \multirow{2}{*}{$-137.2$\%} \\
& & TabM & 32 & -- & 960 & 5 & $4.76$M & \\
\midrule
\multirow{2}{*}{jannis} & \multirow{2}{*}{0} & LoMETab & 16 & 16 & 176 & 4 & $802.6$K & \multirow{2}{*}{$+40.1$\%} \\
& & TabM & 32 & -- & 512 & 5 & $1.34$M & \\
\midrule
\multirow{2}{*}{kdd\_ipums\_la} & \multirow{2}{*}{0} & LoMETab & 32 & 4 & 624 & 2 & $953.2$K & \multirow{2}{*}{$-11.7$\%} \\
& & TabM & 32 & -- & 400 & 5 & $853.5$K & \\
\midrule
\multirow{2}{*}{kdd\_ipums\_la} & \multirow{2}{*}{1} & LoMETab & 16 & 4 & 320 & 4 & $555.4$K & \multirow{2}{*}{$-95.3$\%} \\
& & TabM & 32 & -- & 208 & 5 & $284.4$K & \\
\midrule
\multirow{2}{*}{kdd\_ipums\_la} & \multirow{2}{*}{2} & LoMETab & 32 & 8 & 288 & 1 & $379.7$K & \multirow{2}{*}{$+35.5$\%} \\
& & TabM & 32 & -- & 464 & 3 & $589.1$K & \\
\midrule
\multirow{2}{*}{medical\_charge…} & \multirow{2}{*}{0} & LoMETab & 16 & 16 & 544 & 1 & $176.6$K & \multirow{2}{*}{$+62.5$\%} \\
& & TabM & 32 & -- & 288 & 5 & $471.0$K & \\
\midrule
\multirow{2}{*}{nyc-taxi-green-…} & \multirow{2}{*}{0} & LoMETab & 32 & 16 & 608 & 4 & $3.50$M & \multirow{2}{*}{$-5,104.2$\%} \\
& & TabM & 32 & -- & 80 & 5 & $67.2$K & \\
\midrule
\multirow{2}{*}{particulate-mat…} & \multirow{2}{*}{0} & LoMETab & 32 & 16 & 352 & 2 & $762.0$K & \multirow{2}{*}{$-663.0$\%} \\
& & TabM & 32 & -- & 224 & 2 & $99.9$K & \\
\midrule
\multirow{2}{*}{phoneme} & \multirow{2}{*}{0} & LoMETab & 32 & 2 & 800 & 1 & $201.8$K & \multirow{2}{*}{$+84.7$\%} \\
& & TabM & 32 & -- & 736 & 3 & $1.32$M & \\
\midrule
\multirow{2}{*}{phoneme} & \multirow{2}{*}{1} & LoMETab & 32 & 4 & 752 & 3 & $1.74$M & \multirow{2}{*}{$-2,339.7$\%} \\
& & TabM & 32 & -- & 176 & 2 & $71.5$K & \\
\midrule
\multirow{2}{*}{phoneme} & \multirow{2}{*}{2} & LoMETab & 16 & 16 & 848 & 3 & $2.67$M & \multirow{2}{*}{$-215.1$\%} \\
& & TabM & 32 & -- & 400 & 5 & $847.0$K & \\
\midrule
\multirow{2}{*}{phoneme} & \multirow{2}{*}{3} & LoMETab & 16 & 4 & 448 & 2 & $406.0$K & \multirow{2}{*}{$+91.1$\%} \\
& & TabM & 32 & -- & 1008 & 5 & $4.59$M & \\
\midrule
\multirow{2}{*}{phoneme} & \multirow{2}{*}{4} & LoMETab & 16 & 4 & 400 & 4 & $724.1$K & \multirow{2}{*}{$+75.7$\%} \\
& & TabM & 32 & -- & 928 & 4 & $2.97$M & \\
\midrule
\multirow{2}{*}{pol} & \multirow{2}{*}{0} & LoMETab & 32 & 8 & 416 & 4 & $1.89$M & \multirow{2}{*}{$+46.7$\%} \\
& & TabM & 32 & -- & 880 & 5 & $3.54$M & \\
\midrule
\multirow{2}{*}{pol} & \multirow{2}{*}{1} & LoMETab & 16 & 8 & 176 & 4 & $535.5$K & \multirow{2}{*}{$-5.1$\%} \\
& & TabM & 32 & -- & 432 & 3 & $509.8$K & \\
\midrule
\multirow{2}{*}{road-safety} & \multirow{2}{*}{0} & LoMETab & 16 & 8 & 896 & 3 & $2.78$M & \multirow{2}{*}{$-630.6$\%} \\
& & TabM & 32 & -- & 288 & 4 & $379.9$K & \\
\midrule
\multirow{2}{*}{superconduct} & \multirow{2}{*}{0} & LoMETab & 16 & 16 & 592 & 4 & $4.29$M & \multirow{2}{*}{$-27.8$\%} \\
& & TabM & 32 & -- & 848 & 5 & $3.35$M & \\
\midrule
\multirow{2}{*}{wine} & \multirow{2}{*}{0} & LoMETab & 32 & 2 & 464 & 3 & $810.4$K & \multirow{2}{*}{$+9.6$\%} \\
& & TabM & 32 & -- & 480 & 4 & $896.6$K & \\
\midrule
\multirow{2}{*}{wine} & \multirow{2}{*}{1} & LoMETab & 32 & 4 & 576 & 1 & $204.1$K & \multirow{2}{*}{$+70.7$\%} \\
& & TabM & 32 & -- & 416 & 4 & $697.2$K & \\
\midrule
\multirow{2}{*}{wine} & \multirow{2}{*}{2} & LoMETab & 16 & 8 & 1008 & 1 & $494.7$K & \multirow{2}{*}{$-61.5$\%} \\
& & TabM & 32 & -- & 256 & 4 & $306.3$K & \\
\midrule
\multirow{2}{*}{wine} & \multirow{2}{*}{3} & LoMETab & 16 & 1 & 1008 & 3 & $2.41$M & \multirow{2}{*}{$+9.2$\%} \\
& & TabM & 32 & -- & 752 & 5 & $2.66$M & \\
\midrule
\multirow{2}{*}{wine} & \multirow{2}{*}{4} & LoMETab & 16 & 8 & 912 & 4 & $3.52$M & \multirow{2}{*}{$-296.8$\%} \\
& & TabM & 32 & -- & 832 & 2 & $888.2$K & \\
\midrule
\multirow{2}{*}{wine\_quality} & \multirow{2}{*}{0} & LoMETab & 16 & 16 & 928 & 4 & $4.72$M & \multirow{2}{*}{$-346.8$\%} \\
& & TabM & 32 & -- & 656 & 3 & $1.06$M & \\
\midrule
\multirow{2}{*}{wine\_quality} & \multirow{2}{*}{1} & LoMETab & 32 & 8 & 608 & 3 & $1.92$M & \multirow{2}{*}{$+56.6$\%} \\
& & TabM & 32 & -- & 992 & 5 & $4.42$M & \\
\midrule
\multirow{2}{*}{wine\_quality} & \multirow{2}{*}{2} & LoMETab & 32 & 16 & 848 & 4 & $5.76$M & \multirow{2}{*}{$-194.8$\%} \\
& & TabM & 32 & -- & 640 & 5 & $1.95$M & \\
\midrule
\multirow{2}{*}{year} & \multirow{2}{*}{0} & LoMETab & 32 & 16 & 144 & 4 & $1.79$M & \multirow{2}{*}{$-303.9$\%} \\
& & TabM & 32 & -- & 384 & 3 & $443.0$K & \\
\end{longtable}

%% file: tables/tab_appendix_clf_acc.tex
\begin{table}[htbp]
\centering\small
\caption{Classification accuracy (ACC\,$\uparrow$, 15 seeds, mean\,$\pm$\,std).}
\label{tab:appendix_clf_acc}
\resizebox{\linewidth}{!}{%
\begin{tabular}{ll|r|r|r|r|r}
\toprule
Model & Param & \textbf{HI} & \textbf{OT} & \textbf{CO} & \textbf{AD} & \textbf{CH} \\
\midrule
\multirow{5}{*}{\shortstack[l]{LoMETab \\ \scriptsize($K{=}16,\,\sigma{=}1.0$)}} & $r{=1}$ & 0.724{\scriptsize$\pm$0.001} & 0.803{\scriptsize$\pm$0.002} & \textbf{0.958{\scriptsize$\pm$0.001}} & 0.859{\scriptsize$\pm$0.002} & 0.859{\scriptsize$\pm$0.003} \\
 & $r{=2}$ & 0.728{\scriptsize$\pm$0.002} & 0.812{\scriptsize$\pm$0.002} & 0.932{\scriptsize$\pm$0.001} & 0.860{\scriptsize$\pm$0.001} & 0.860{\scriptsize$\pm$0.003} \\
 & $r{=4}$ & 0.730{\scriptsize$\pm$0.001} & 0.823{\scriptsize$\pm$0.002} & 0.939{\scriptsize$\pm$0.002} & 0.859{\scriptsize$\pm$0.001} & 0.860{\scriptsize$\pm$0.003} \\
 & $r{=8}$ & 0.730{\scriptsize$\pm$0.002} & 0.824{\scriptsize$\pm$0.002} & 0.948{\scriptsize$\pm$0.001} & 0.859{\scriptsize$\pm$0.001} & 0.860{\scriptsize$\pm$0.002} \\
 & $r{=16}$ & \textbf{0.730{\scriptsize$\pm$0.003}} & \textbf{0.825{\scriptsize$\pm$0.002}} & 0.956{\scriptsize$\pm$0.001} & 0.859{\scriptsize$\pm$0.002} & 0.859{\scriptsize$\pm$0.003} \\
\midrule
\multirow{5}{*}{\shortstack[l]{LoMETab \\ \scriptsize($K{=}16,\,r{=}16$)}} & $\sigma{=0.1}$ & 0.721{\scriptsize$\pm$0.003} & 0.800{\scriptsize$\pm$0.004} & 0.956{\scriptsize$\pm$0.001} & 0.859{\scriptsize$\pm$0.001} & 0.860{\scriptsize$\pm$0.002} \\
 & $\sigma{=0.3}$ & 0.722{\scriptsize$\pm$0.003} & 0.800{\scriptsize$\pm$0.002} & 0.957{\scriptsize$\pm$0.001} & 0.859{\scriptsize$\pm$0.002} & 0.858{\scriptsize$\pm$0.003} \\
 & $\sigma{=0.5}$ & 0.728{\scriptsize$\pm$0.002} & 0.806{\scriptsize$\pm$0.002} & 0.950{\scriptsize$\pm$0.024} & 0.859{\scriptsize$\pm$0.001} & \textbf{0.861{\scriptsize$\pm$0.002}} \\
 & $\sigma{=1.0}$ & \textbf{0.730{\scriptsize$\pm$0.003}} & \textbf{0.825{\scriptsize$\pm$0.002}} & 0.956{\scriptsize$\pm$0.001} & 0.859{\scriptsize$\pm$0.002} & 0.859{\scriptsize$\pm$0.003} \\
 & $\sigma{=2.0}$ & 0.728{\scriptsize$\pm$0.002} & 0.825{\scriptsize$\pm$0.002} & 0.955{\scriptsize$\pm$0.002} & 0.859{\scriptsize$\pm$0.001} & 0.857{\scriptsize$\pm$0.003} \\
\midrule
\multirow{4}{*}{TabM} & $K{=16}$ & 0.729{\scriptsize$\pm$0.001} & 0.822{\scriptsize$\pm$0.001} & 0.937{\scriptsize$\pm$0.006} & 0.861{\scriptsize$\pm$0.001} & 0.845{\scriptsize$\pm$0.002} \\
 & $K{=32}$ & 0.729{\scriptsize$\pm$0.002} & 0.822{\scriptsize$\pm$0.002} & 0.933{\scriptsize$\pm$0.006} & \textbf{0.861{\scriptsize$\pm$0.002}} & 0.846{\scriptsize$\pm$0.002} \\
 & $K{=64}$ & 0.729{\scriptsize$\pm$0.001} & 0.820{\scriptsize$\pm$0.003} & 0.865{\scriptsize$\pm$0.087} & 0.854{\scriptsize$\pm$0.002} & 0.834{\scriptsize$\pm$0.013} \\
 & $K{=128}$ & 0.729{\scriptsize$\pm$0.001} & 0.816{\scriptsize$\pm$0.003} & 0.874{\scriptsize$\pm$0.086} & 0.860{\scriptsize$\pm$0.007} & 0.841{\scriptsize$\pm$0.012} \\
\bottomrule
\end{tabular}}
\end{table}

%% file: tables/tab_appendix_reg.tex
\begin{table}[htbp]
\centering\small
\caption{Regression RMSE ($\downarrow$, 15 seeds, mean\,$\pm$\,std).}
\label{tab:appendix_reg}
\resizebox{\linewidth}{!}{%
\begin{tabular}{ll|r|r|r|r}
\toprule
Model & Param & \textbf{BL} & \textbf{CA} & \textbf{DI} & \textbf{MI} \\
\midrule
\multirow{5}{*}{\shortstack[l]{LoMETab \\ \scriptsize($K{=}16,\,\sigma{=}1.0$)}} & $r{=1}$ & 0.6895{\scriptsize$\pm$0.0009} & 0.4558{\scriptsize$\pm$0.0049} & 0.1335{\scriptsize$\pm$0.0011} & 0.7507{\scriptsize$\pm$0.0011} \\
 & $r{=2}$ & 0.6852{\scriptsize$\pm$0.0008} & 0.4480{\scriptsize$\pm$0.0033} & 0.1331{\scriptsize$\pm$0.0013} & 0.7500{\scriptsize$\pm$0.0008} \\
 & $r{=4}$ & 0.6851{\scriptsize$\pm$0.0010} & 0.4463{\scriptsize$\pm$0.0016} & 0.1334{\scriptsize$\pm$0.0013} & 0.7485{\scriptsize$\pm$0.0007} \\
 & $r{=8}$ & 0.6842{\scriptsize$\pm$0.0005} & 0.4466{\scriptsize$\pm$0.0028} & 0.1337{\scriptsize$\pm$0.0013} & 0.7471{\scriptsize$\pm$0.0004} \\
 & $r{=16}$ & 0.6835{\scriptsize$\pm$0.0004} & 0.4459{\scriptsize$\pm$0.0033} & 0.1333{\scriptsize$\pm$0.0012} & 0.7468{\scriptsize$\pm$0.0004} \\
\midrule
\multirow{5}{*}{\shortstack[l]{LoMETab \\ \scriptsize($K{=}16,\,r{=}16$)}} & $\sigma{=0.1}$ & 0.6892{\scriptsize$\pm$0.0010} & 0.4714{\scriptsize$\pm$0.0053} & 0.1350{\scriptsize$\pm$0.0010} & 0.7504{\scriptsize$\pm$0.0006} \\
 & $\sigma{=0.3}$ & 0.6892{\scriptsize$\pm$0.0008} & 0.4655{\scriptsize$\pm$0.0052} & 0.1343{\scriptsize$\pm$0.0010} & 0.7503{\scriptsize$\pm$0.0005} \\
 & $\sigma{=0.5}$ & \textbf{0.6833{\scriptsize$\pm$0.0006}} & 0.4508{\scriptsize$\pm$0.0024} & 0.1330{\scriptsize$\pm$0.0010} & 0.7493{\scriptsize$\pm$0.0010} \\
 & $\sigma{=1.0}$ & 0.6835{\scriptsize$\pm$0.0004} & 0.4459{\scriptsize$\pm$0.0033} & 0.1333{\scriptsize$\pm$0.0012} & 0.7468{\scriptsize$\pm$0.0004} \\
 & $\sigma{=2.0}$ & 0.6836{\scriptsize$\pm$0.0011} & \textbf{0.4418{\scriptsize$\pm$0.0045}} & 0.1347{\scriptsize$\pm$0.0019} & 0.7597{\scriptsize$\pm$0.0005} \\
\midrule
\multirow{4}{*}{TabM} & $K{=16}$ & 0.6882{\scriptsize$\pm$0.0015} & 0.4631{\scriptsize$\pm$0.0037} & 0.1357{\scriptsize$\pm$0.0014} & 0.7508{\scriptsize$\pm$0.0008} \\
 & $K{=32}$ & 0.6869{\scriptsize$\pm$0.0008} & 0.4425{\scriptsize$\pm$0.0032} & \textbf{0.1326{\scriptsize$\pm$0.0011}} & \textbf{0.7451{\scriptsize$\pm$0.0004}} \\
 & $K{=64}$ & 0.6878{\scriptsize$\pm$0.0010} & 0.4619{\scriptsize$\pm$0.0042} & 0.1354{\scriptsize$\pm$0.0010} & 0.7513{\scriptsize$\pm$0.0007} \\
 & $K{=128}$ & 0.6881{\scriptsize$\pm$0.0009} & 0.4622{\scriptsize$\pm$0.0042} & 0.1351{\scriptsize$\pm$0.0007} & 0.7513{\scriptsize$\pm$0.0010} \\
\bottomrule
\end{tabular}}
\end{table}

%% file: tables/tab_appendix_clf_ece.tex
\begin{table}[htbp]
\centering\small
\caption{Classification calibration error (ECE\,$\downarrow$, 15 seeds, mean\,$\pm$\,std).}
\label{tab:appendix_clf_ece}
\resizebox{\linewidth}{!}{%
\begin{tabular}{ll|r|r|r|r|r}
\toprule
Model & Param & \textbf{HI} & \textbf{OT} & \textbf{CO} & \textbf{AD} & \textbf{CH} \\
\midrule
\multirow{5}{*}{\shortstack[l]{LoMETab \\ \scriptsize($K{=}16,\,\sigma{=}1.0$)}} & $r{=1}$ & 0.017{\scriptsize$\pm$0.010} & 0.031{\scriptsize$\pm$0.014} & 0.018{\scriptsize$\pm$0.001} & 0.012{\scriptsize$\pm$0.007} & 0.023{\scriptsize$\pm$0.007} \\
 & $r{=2}$ & 0.016{\scriptsize$\pm$0.008} & 0.025{\scriptsize$\pm$0.008} & 0.044{\scriptsize$\pm$0.002} & 0.013{\scriptsize$\pm$0.005} & 0.021{\scriptsize$\pm$0.007} \\
 & $r{=4}$ & 0.015{\scriptsize$\pm$0.007} & 0.011{\scriptsize$\pm$0.002} & 0.041{\scriptsize$\pm$0.003} & 0.012{\scriptsize$\pm$0.006} & 0.021{\scriptsize$\pm$0.007} \\
 & $r{=8}$ & 0.018{\scriptsize$\pm$0.007} & 0.011{\scriptsize$\pm$0.001} & 0.037{\scriptsize$\pm$0.002} & 0.011{\scriptsize$\pm$0.006} & 0.022{\scriptsize$\pm$0.009} \\
 & $r{=16}$ & 0.014{\scriptsize$\pm$0.007} & \textbf{0.009{\scriptsize$\pm$0.002}} & 0.030{\scriptsize$\pm$0.001} & 0.011{\scriptsize$\pm$0.005} & 0.020{\scriptsize$\pm$0.007} \\
\midrule
\multirow{5}{*}{\shortstack[l]{LoMETab \\ \scriptsize($K{=}16,\,r{=}16$)}} & $\sigma{=0.1}$ & 0.020{\scriptsize$\pm$0.010} & 0.041{\scriptsize$\pm$0.019} & 0.017{\scriptsize$\pm$0.002} & 0.010{\scriptsize$\pm$0.005} & 0.025{\scriptsize$\pm$0.008} \\
 & $\sigma{=0.3}$ & 0.017{\scriptsize$\pm$0.008} & 0.040{\scriptsize$\pm$0.014} & 0.017{\scriptsize$\pm$0.001} & \textbf{0.010{\scriptsize$\pm$0.004}} & \textbf{0.020{\scriptsize$\pm$0.006}} \\
 & $\sigma{=0.5}$ & 0.019{\scriptsize$\pm$0.008} & 0.033{\scriptsize$\pm$0.016} & 0.026{\scriptsize$\pm$0.003} & 0.012{\scriptsize$\pm$0.005} & 0.025{\scriptsize$\pm$0.009} \\
 & $\sigma{=1.0}$ & 0.014{\scriptsize$\pm$0.007} & \textbf{0.009{\scriptsize$\pm$0.002}} & 0.030{\scriptsize$\pm$0.001} & 0.011{\scriptsize$\pm$0.005} & 0.020{\scriptsize$\pm$0.007} \\
 & $\sigma{=2.0}$ & 0.013{\scriptsize$\pm$0.008} & 0.013{\scriptsize$\pm$0.004} & 0.034{\scriptsize$\pm$0.002} & 0.012{\scriptsize$\pm$0.006} & 0.020{\scriptsize$\pm$0.008} \\
\midrule
\multirow{3}{*}{TabM} & $K{=32}$ & 0.022{\scriptsize$\pm$0.009} & 0.043{\scriptsize$\pm$0.010} & \textbf{0.015{\scriptsize$\pm$0.003}} & 0.024{\scriptsize$\pm$0.011} & 0.022{\scriptsize$\pm$0.007} \\
 & $K{=64}$ & \textbf{0.013{\scriptsize$\pm$0.004}} & 0.030{\scriptsize$\pm$0.011} & 0.016{\scriptsize$\pm$0.005} & 0.011{\scriptsize$\pm$0.004} & 0.084{\scriptsize$\pm$0.065} \\
 & $K{=128}$ & 0.014{\scriptsize$\pm$0.007} & 0.020{\scriptsize$\pm$0.005} & 0.016{\scriptsize$\pm$0.004} & 0.016{\scriptsize$\pm$0.008} & 0.047{\scriptsize$\pm$0.052} \\
\bottomrule
\end{tabular}}
\end{table}

%% file: tables/per-dataset-why.tex
\newcommand{\topalign}[1]{%
\vtop{\vskip 0pt #1}}

\begin{longtable}{p{0.5\textwidth}p{0.5\textwidth}}
\caption{
Extended results for the Tabzilla benchmark.
}\\
\end{longtable}

%% file: tables/per-dataset-default.tex
\providecommand{\topalign}[1]{\vtop{\vskip 0pt #1}}

\begin{longtable}{p{0.5\textwidth}p{0.5\textwidth}}
\caption{
Extended results for the default benchmark.
}\\
\end{longtable}

%% file: references.bib
@inproceedings{wen2020batchensemble,
  title     = {{BatchEnsemble}: An Alternative Approach to Efficient Ensemble and Lifelong Learning},
  author    = {Wen, Yeming and Tran, Dustin and Ba, Jimmy},
  booktitle = {International Conference on Learning Representations (ICLR)},
  year      = {2020},
  url       = {https://openreview.net/forum?id=Sklf1yrYDr}
}

@inproceedings{durasov2021masksembles,
  title     = {Masksembles for Uncertainty Estimation},
  author    = {Durasov, Nikita and Bagautdinov, Timur and Baque, Pierre and Fua, Pascal},
  booktitle = {IEEE/CVF Conference on Computer Vision and Pattern Recognition (CVPR)},
  year      = {2021}
}

@inproceedings{turkoglu2022filmensemble,
  title     = {{FiLM-Ensemble}: Probabilistic Deep Learning via Feature-wise Linear Modulation},
  author    = {Turkoglu, Mehmet Ozgur and Becker, Alexander and G{\"u}nd{\"u}z, H{\"u}seyin Anil and Rezaei, Mina and Bischl, Bernd and Daudt, Rodrigo Caye and D'Aronco, Stefano and Wegner, Jan Dirk and Schindler, Konrad},
  booktitle = {Advances in Neural Information Processing Systems (NeurIPS)},
  year      = {2022},
  url       = {https://arxiv.org/abs/2206.00050}
}

@inproceedings{nam2021diversity,
  title     = {Diversity Matters When Learning From Ensembles},
  author    = {Nam, Giung and Yoon, Jongmin and Lee, Yoonho and Lee, Juho},
  booktitle = {Advances in Neural Information Processing Systems (NeurIPS)},
  year      = {2021},
  url       = {https://openreview.net/forum?id=f_eOQN87eXc}
}

@inproceedings{gorishniy2025tabm,
  title     = {{TabM}: Advancing Tabular Deep Learning with Parameter-Efficient Ensembling},
  author    = {Gorishniy, Yury and Kotelnikov, Akim and Babenko, Artem},
  booktitle = {International Conference on Learning Representations (ICLR)},
  year      = {2025},
  url       = {https://openreview.net/forum?id=Sd4wYYOhmY}
}

@article{zamyatin2026batchensemble,
  title   = {Is {BatchEnsemble} a Single Model? {O}n Calibration and Diversity of Efficient Ensembles},
  author  = {Zamyatin, Anton and Indri, Patrick and Malhotra, Sagar and G{\"a}rtner, Thomas},
  journal = {arXiv preprint arXiv:2601.16936},
  year    = {2026},
  url     = {https://arxiv.org/abs/2601.16936}
}

@inproceedings{hu2022lora,
  title     = {{LoRA}: Low-Rank Adaptation of Large Language Models},
  author    = {Hu, Edward J. and Shen, Yelong and Wallis, Phillip and Allen-Zhu, Zeyuan and Li, Yuanzhi and Wang, Shean and Wang, Lu and Chen, Weizhu},
  booktitle = {International Conference on Learning Representations (ICLR)},
  year      = {2022},
  url       = {https://openreview.net/forum?id=nZeVKeeFYf9}
}

@article{halbheer2024lora,
  title   = {{LoRA-Ensemble}: Efficient Uncertainty Modelling for Self-Attention Networks},
  author  = {Halbheer, Michelle and M{\"u}hlematter, Dominik Jan and Becker, Alexander and Narnhofer, Dominik and Aasen, Helge and Schindler, Konrad and Turkoglu, Mehmet Ozgur},
  journal = {arXiv preprint arXiv:2405.14438},
  year    = {2024},
  url     = {https://arxiv.org/abs/2405.14438}
}

@inproceedings{huang2025hira,
  title     = {{HiRA}: Parameter-Efficient Hadamard High-Rank Adaptation for Large Language Models},
  author    = {Huang, Qiushi and Ko, Tom and Zhuang, Zhan and Tang, Lilian and Zhang, Yu},
  booktitle = {International Conference on Learning Representations (ICLR)},
  year      = {2025},
  note      = {Oral presentation},
  url       = {https://openreview.net/forum?id=TwJrTz9cRS}
}

@inproceedings{gorishniy2022embeddings,
  title     = {On Embeddings for Numerical Features in Tabular Deep Learning},
  author    = {Gorishniy, Yury and Rubachev, Ivan and Babenko, Artem},
  booktitle = {Advances in Neural Information Processing Systems (NeurIPS)},
  year      = {2022}
}

@inproceedings{gorishniy2024modernnca,
  title     = {{TabR}: Tabular Deep Learning Meets Nearest Neighbors},
  author    = {Gorishniy, Yury and Rubachev, Ivan and Kartashev, Nikolay and Shlenskii, Daniil and Kotelnikov, Akim and Babenko, Artem},
  booktitle = {International Conference on Learning Representations (ICLR)},
  year      = {2024}
}

@inproceedings{grinsztajn2022tree,
  title     = {Why do Tree-Based Models Still Outperform Deep Learning on Typical Tabular Data?},
  author    = {Grinsztajn, L{\'e}o and Oyallon, Edouard and Varoquaux, Ga{\"e}l},
  booktitle = {Advances in Neural Information Processing Systems (NeurIPS) Datasets and Benchmarks Track},
  year      = {2022}
}

@inproceedings{mcelfresh2024neural,
  title     = {When Do Neural Nets Outperform Boosted Trees on Tabular Data?},
  author    = {McElfresh, Duncan and Khandagale, Sujay and Valverde, Jonathan and Prasad C, Vishak and Ramakrishnan, Ganesh and Goldblum, Micah and White, Colin},
  booktitle = {Advances in Neural Information Processing Systems (NeurIPS) Datasets and Benchmarks Track},
  year      = {2024}
}

@inproceedings{rubachev2024tabred,
  title     = {{TabReD}: Analyzing Pitfalls and Filling the Gaps in Tabular Deep Learning Benchmarks},
  author    = {Rubachev, Ivan and Kartashev, Nikolay and Gorishniy, Yury and Babenko, Artem},
  booktitle = {International Conference on Learning Representations (ICLR)},
  year      = {2025}
}

@inproceedings{chen2016xgboost,
  title     = {{XGBoost}: A Scalable Tree Boosting System},
  author    = {Chen, Tianqi and Guestrin, Carlos},
  booktitle = {ACM SIGKDD International Conference on Knowledge Discovery and Data Mining (KDD)},
  year      = {2016}
}

@inproceedings{prokhorenkova2018catboost,
  title     = {{CatBoost}: Unbiased Boosting with Categorical Features},
  author    = {Prokhorenkova, Liudmila and Gusev, Gleb and Vorobev, Aleksandr and Dorogush, Anna Veronika and Gulin, Andrey},
  booktitle = {Advances in Neural Information Processing Systems (NeurIPS)},
  year      = {2018}
}

@inproceedings{akiba2019optuna,
  title     = {{Optuna}: A Next-Generation Hyperparameter Optimization Framework},
  author    = {Akiba, Takuya and Sano, Shotaro and Yanase, Toshihiko and Ohta, Takeru and Koyama, Masanori},
  booktitle = {ACM SIGKDD International Conference on Knowledge Discovery and Data Mining (KDD)},
  year      = {2019}
}

@inproceedings{lakshminarayanan2017simple,
  title={Simple and Scalable Predictive Uncertainty Estimation using Deep Ensembles},
  author={Lakshminarayanan, Balaji and Pritzel, Alexander and Blundell, Charles},
  booktitle={Advances in Neural Information Processing Systems (NeurIPS)},
  year={2017}
}

@inproceedings{he2015delving,
  title={Delving deep into rectifiers: Surpassing human-level performance on imagenet classification},
  author={He, Kaiming and Zhang, Xiangyu and Ren, Shaoqing and Sun, Jian},
  booktitle={Proceedings of the IEEE international conference on computer vision},
  pages={1026--1034},
  year={2015}
}

@inproceedings{krogh1994neural,
  title     = {Neural Network Ensembles, Cross Validation, and Active Learning},
  author    = {Krogh, Anders and Vedelsby, Jesper},
  booktitle = {Advances in Neural Information Processing Systems},
  volume    = {7},
  pages     = {231--238},
  year      = {1995},
  publisher = {MIT Press},
  url       = {https://papers.nips.cc/paper_files/paper/1994/hash/b8c37e33defde51cf91e1e03e51657da-Abstract.html}
}

@book{horn2012matrix,
  title={Matrix Analysis},
  author={Horn, Roger A. and Johnson, Charles R.},
  year={2012},
  edition={2nd},
  publisher={Cambridge University Press},
  address={Cambridge}
}

@inproceedings{naeini2015obtaining,
  title={Obtaining well calibrated probabilities using {B}ayesian binning},
  author={Naeini, Mahdi Pakdaman and Cooper, Gregory and Hauskrecht, Milos},
  booktitle={Proceedings of the AAAI Conference on Artificial Intelligence},
  volume={29},
  number={1},
  year={2015}
}

@inproceedings{guo2017calibration,
  title={On calibration of modern neural networks},
  author={Guo, Chuan and Pleiss, Geoff and Sun, Yu and Weinberger, Kilian Q},
  booktitle={International Conference on Machine Learning},
  pages={1321--1330},
  year={2017},
  organization={PMLR}
}
